\documentclass{article}

\usepackage{preamble}
\usepackage{preamblelupi}

\usepackage[accepted]{icml2013v2}
\numberwithin{equation}{section}

\icmltitlerunning{Learning Using Privileged Information: SVM+ and Weighted SVM}
\icmlnoticestring{\textit{
$^{\rm a}$Max Planck Institute for Informatics, Saarbr\"{u}cken, Germany.
$^{\rm b}$Saarland University, Saarbr\"{u}cken, Germany.
}}

\begin{document}

\twocolumn[
\icmltitle{Learning Using Privileged Information: SVM+ and Weighted SVM}

\icmlauthor{Maksim Lapin\texorpdfstring{$^{\rm a}$}{}}%
{mlapin@mpi-inf.mpg.de}
\icmlauthor{Matthias Hein\texorpdfstring{$^{\rm b}$}{}}%
{hein@cs.uni-saarland.de}
\icmlauthor{Bernt Schiele\texorpdfstring{$^{\rm a}$}{}}%
{schiele@mpi-inf.mpg.de}

\icmlkeywords{SVM, SVM+, Weighted SVM, Importance weighting,
Privileged information, Prior knowledge}

\vskip 0.3in
]

\begin{abstract}
Prior knowledge can be used to improve predictive performance
of learning algorithms or reduce the amount of data required for training.
The same goal is pursued within the learning using privileged information
paradigm which was recently introduced by Vapnik et al.\ and
is aimed at utilizing additional information available only at training time
-- a framework implemented by SVM+.
We relate the privileged information to importance weighting
and show that the prior knowledge expressible with privileged
features can also be encoded by weights associated with every
training example.
We show that a weighted SVM can always replicate an SVM+ solution,
while the converse is not true and we construct a counterexample
highlighting the limitations of SVM+.
Finally, we touch on the problem of choosing weights for
weighted SVMs when privileged features are not available.
\end{abstract}

\section{Introduction: prior knowledge, privileged information,
and instance weights}
\label{sec:intro}

Classification is a well-studied problem in machine learning,
however, learning still remains a challenging task when the amount of
training data is limited. Hence, information available \emph{in addition}
to the training sample -- the prior knowledge -- is the crucial factor
in achieving further performance improvement.

Prior knowledge comes in different forms and its incorporation
into the learning problem depends on a particular setting as well as
the algorithm. This paper focuses on introducing prior knowledge
into a support vector machine (SVM) for binary classification.
\citet{lauer:2008:prior} provide a review
of different ways to incorporate prior knowledge into SVMs
and give a categorization of the reviewed methods
based on the \emph{type} of prior knowledge they assume;
see also \citep{schoelkopf:2002:learning}.
We will mainly consider the scenario where the additional information
is about the \emph{training data} rather than about the target function.
A loosely related setting is the semi-supervised learning approach
\citep{chapelle:2006:semisup}, where unlabeled data carries certain
information about the marginal distribution in the input space.

Recently, \citet{vapnik:2009:lupi} introduced the learning using privileged
information (LUPI) paradigm which aims at improving predictive
performance of learning algorithms and reducing the amount of required
training data. The additional information in this framework comes in
the form of privileged features, which are available at training time,
but not at test time. These features are used to parametrize the upper
bound on the loss function and, essentially, are used to estimate
the loss of an optimal classifier on the given training sample.
Higher loss may be seen as an indication that a given point
is likely to be an outlier, and, hence, should be treated differently
than a non-outlier. This simple idea has been extensively explored
in the literature and we give a few pointers in Section~\ref{sec:related}.
The additional information about which training examples
are likely to be outliers can be encoded via instance weights,
therefore, one can already anticipate a close relation between
the LUPI framework and importance weighting which is discussed next.

In the weighted learning scenario, each training example comes
with a non-negative weight which is used in the loss function
to balance the cost of errors.
A typical example where instance weights appear naturally
is the cost-sensitive learning \citep{elkan:2001:costsens}.
If classes are unbalanced or
different misclassification errors incur different penalties,
one can encode that prior knowledge in the form of instance weights.
Assigning high weight to a data point suggests that the learning
algorithm should try to classify that point correctly, possibly
at the cost of misclassifying ``less important'' points.
In this paper, however, we do \emph{not} make the cost-sensitive
assumption, i.e., we do not assume that different errors incur
different costs on the \emph{test} set. Instead, we decouple
importance weighting on the training and on the test sets,
and we only focus on the former. This allows us, in particular,
to also assign a \emph{high} weight to an outlier
if that ultimately leads to a better model.

As mentioned above, there are different forms of prior knowledge
that can be encoded differently. In this paper, we show that
instance weights can express \emph{the same type of prior knowledge}
that is encoded via privileged features.
In particular, this allows one to interpret the effect of
privileged features in terms of the incurred importance weights.
Remarkably, the resulting weights \emph{do} emphasize outliers,
which also happen to be support vectors in SVMs.

Our focus in this work is on the study of the SVM$+$ algorithm,
which is an extension of the support vector machine to the
LUPI framework \citep{vapnik:2009:lupi}.
Using basic tools of convex analysis, we investigate
uniqueness of the SVM$+$ solution and its relation to solutions
of the weighted SVM (WSVM). It turns out there is a simple connection
between an SVM$+$ solution and WSVM instance weights, moreover,
that relation can be used to better understand the SVM$+$ algorithm
and to study its limitations. Having realized that instance weights
in WSVMs can serve the same purpose as privileged features in SVM$+$,
we also turn to the problem of choosing weights when privileged
features are not available.

\subsection{Our contributions}
\label{sec:contrib}

Below is a summary of contributions of this work.

\begin{itemize}

\item We show that any non-trivial SVM$+$ solution is unique
(in the primal), which is a stronger result than the one available
for (W)SVMs, where the offset $b$ may not be unique.

\item By reformulating the SVM$+$ dual optimization problem, we reveal
its close connection to the WSVM algorithm. In particular, we show
that any SVM$+$ dual solution can be used to construct
weights for the WSVM that will yield the same primal solution
up to the non-uniqueness of $b$. This implies that WSVM
with appropriately chosen weights can mimic SVM$+$ and that it is
always possible to go from an SVM$+$ solution to a WSVM solution.

\item We also study whether it is always possible to go in the opposite
direction (which would imply that the two algorithms are equivalent).
We give the necessary and sufficient condition for such an equivalence
to hold and reveal that the SVM$+$ solutions are a strict subset
of the WSVM solutions.
We construct a simple counterexample where a WSVM solution
cannot be found by SVM$+$, no matter which privileged features
are used or which values the hyper-parameters take.

\item Finally, we turn to the problem of choosing weights in the
absence of privileged features. We show that the weights can be learned
directly from data by minimizing an estimate of risk similar to
standard procedures of hyper-parameter tuning. In the idealized
setting, where the estimate is computed on a large validation set,
we show that the WSVM with learned weights outperforms both the SVM
and the SVM$+$. This highlights the potential of weighted learning
and should motivate further work on the choice of weights.

\end{itemize}

\subsection{Related work}
\label{sec:related}

We now briefly discuss related work on learning using privileged information
and weighted learning.

Since the introduction of the new learning paradigm
and the corresponding SVM$+$ algorithm in
\citep{vapnik:2006:afterword} and later in
\citep{vapnik:2009:luhi,vapnik:2009:lupi},
there is a growing body of work on theoretical analysis
\citep{pechyony:2010:tlpi}, implementation \citep{pechyony:2011:fastopt}
and application of the proposed framework to various machine learning
settings.
\citet{liang:2008:svmpmtl, liang:2009:structdata} study the relation
between the SVM$+$ approach and the multi-task learning scenario,
\citet{fouad:2012:gmlvq} apply the SVM$+$ idea to metric learning,
and \citet{chen:2012:boosting} extend it to boosting algorithms.
\citet{feyereis:2012:clustering} use privileged information for
data clustering and \citet{wolf:2013:svmminus} propose
an SVM$\ominus$ method to compute similarity scores in
video face recognition. Note, however, that the latter method is not
related to the SVM$-$ algorithm we have in mind
in Section~\ref{sec:svmminus}. In particular,
SVM$\ominus$ reduces to SVM with a pre-processing step,
similar to \citep{schoelkopf:1998:prior}, while in our case
the optimization problem as well as the motivation are entirely different.

Instance weighting has been widely used in various machine learning settings
and the topic is to too vast to cover all of the related work here.
We only give a few pointers to papers on
cost-sensitive learning
\citep{margineantu:2002:classprob, zadrozny:2003:costsens},
sample bias correction \citep{heckman:1979:sample, cortes:2010:bounds},
domain adaptation \citep{shimodaira:2000:covshift, sugiyama:2005:covshift},
online learning \citep{dredze:2008:confidence},
and active learning \citep{beygelzimer:2009}.
Perhaps the most related in terms of the learning algorithm (SVM)
and the \emph{interpretation} of instance weights are the works on
fuzzy SVM \citep{lin:2002:fuzzysvm}, where each data point has
a fuzzy class membership represented by a weight between 0 and 1,
weighted margin SVM \citep{wu:2004:wmsvmpriorknow},
where again each label has a confidence score between 0 and 1,
and weighted SVM with an outlier detection pre-processing step
\citep{yang:2005:wsvm}, where a kernel-based clustering algorithm
is used to generate instance weights.

\subsection{Organization}
\label{sec:organization}

The rest of the paper is organized as follows.
In Section~\ref{sec:preliminaries}
we introduce the SVM$+$ and the weighted SVM (WSVM) algorithms.
In Section~\ref{sec:uniqueness}
we study basic properties of these algorithms, namely,
uniqueness of their solutions.
In Section~\ref{sec:svmp_svmw_relation} we present our main result
which consists of four parts.
Theorem~\ref{thm:svmw_from_svmp} shows that any SVM$+$ solution is
also a WSVM solution with appropriately chosen weights,
Theorem~\ref{thm:svmw_to_svmp_iff} gives the necessary and sufficient
condition for equivalence between the SVM$+$ and WSVM problems, and
Section~\ref{sec:example_svmw_not_svmp} presents an example
where a WSVM solution cannot be found by SVM$+$,
no matter which privileged features are used.
Finally, Section~\ref{sec:svmminus} discusses whether it is possible
to complement SVM$+$ with an SVM$-$.

Section~\ref{sec:choose_weights} is concerned with the problem of
choosing weights, where we propose a weight learning method
in Section~\ref{sec:weight_learning}.
Lastly, Section~\ref{sec:experiments} presents experimental results
on a number of publicly available data sets
and Section~\ref{sec:conclusion} gives some concluding remarks.

All proofs are moved to Appendix to enhance readability.

\section{Preliminaries}
\label{sec:preliminaries}

In this section we describe the necessary background.
Our results are based on basic notions from convex analysis
\citep{boyd:2004:cvx}
and, in particular, on the Karush-Kuhn-Tucker (KKT) conditions.
For convenience, the latter are provided in Appendix~\ref{sec:kkt}
for both of the optimization problems studied below.

\subsection{The setting and notation}
\label{sec:notation}

We consider a binary classification problem with an instance space $\Xc$
and the label set $\Yc = \{-1,1\}$.
Let $S = \{(\vx_i, y_i)\}_{i=1}^{n}$ be a training sample drawn
i.i.d.\ from an unknown distribution $\Prob$ on $\Xc \times \Yc$,
and $\loss$ be a convex loss function $\loss : \Rb \rightarrow \Rb_+$,
e.g., the hinge loss $\loss(y f(\vx)) = [1 - y f(\vx)]_{+}$.
The task is to learn $f : \Xc \rightarrow \Rb$
from a set of hypotheses $\Hc$,
that yields label prediction by $\sign f(\vx)$
and achieves the lowest expected loss $L(f) \bydef \Exp \loss(Y f(X))$.

We use $\cXc$ to denote
the space of privileged information used in the SVM$+$,
while the $^\star$ is reserved to indicate a solution
to an optimization problem.

In the non-linear setting, the input data is first mapped
into a feature space endowed with an inner product.
The decision space $\Xc$ is mapped into $\Zc$ via a feature map $\Phi$
($\vx_i \mapsto \Phi(\vx_i) = \vz_i$) and
the correcting space $\cXc$ is mapped into $\cZc$ via $\cPhi$
($\cvx_i \mapsto \cPhi(\cvx_i) = \cvz_i$).
It is known \citep{schoelkopf:2001:representer} that inner
products correspond to positive definite kernel functions\footnote{
A function $k: \Xc \times \Xc \rightarrow \Rb$ which for all $n \in \Nb$,
$\vx_1, \ldots, \vx_n \in \Xc$ gives rise to a positive definite
kernel matrix $\mK$ is called a positive definite kernel.}
as follows:
$\inner{\vz_i, \vz_j}_{\Zc}
= \inner{\Phi(\vx_i), \Phi(\vx_j)}_{\Zc}
= k(\vx_i, \vx_j)$
(and similar for $\cXc$),
which allows to formulate algorithms with general kernels in mind.
Since the corresponding space should be clear from the context,
we omit the subscripts when dealing with inner products
and the induced norms.

Unless transposed with $^\top$, all vectors are column vectors
denoted by lower case bold letters, matrices are denoted by
capital bold letters, and random variables are denoted by capital letters.
We let $\vy = (y_1, \ldots, y_n)^{\top}$ and $\mY = \diag(\vy)$.
The kernel matrices $\mK$ and $\cmK$ are defined entrywise via
$\mK_{ij} = k(\vx_i, \vx_j)$ and $\cmK_{ij} = \ck(\cvx_i, \cvx_j)$,
where $i,j = 1, \ldots, n$.
We also introduce the index sets
$\Ic_{\pm} \bydef \{i : y_i \gtrless 0\}$, $\Ic_{0} \bydef \{i : y_i
f(\vx_i) < 1\}$, $\Ic_{1} \bydef \{i : y_i f(\vx_i) \leq 1\}$,
and a shorthand $\Prob(1|\vx) \bydef \Prob(Y=1|X=\vx)$.

Finally, $\Null(\mA)$ and $\Range(\mA)$ stand correspondingly for the
null space and the column space of a matrix $\mA$,
$\va^{\bot}$ is the orthogonal complement of $\va$,
and $\zeros$ (respectively $\ones$) is the vector of all zeros (ones).

\subsection{The SVM+ optimization problem}
\label{sec:background_svmp}

In the framework of learning using privileged information (LUPI), the decision
space $\Xc$ is augmented with a correcting space $\cXc$ of privileged features
$\cvx$ that are available \emph{at training time only} and are essentially used
to estimate the loss $\loss(y_i f^{\star}(\vx_i))$ of an optimal classifier
$f^\star \bydef \argmin_{f \in \Hc} L(f)$ on the given training sample.
The SVM$+$ algorithm \citep{pechyony:2011:fastopt} is a generalization of
the support vector machine that implements the LUPI paradigm. The slack
variables $\xi_i$ are parametrized as a function of privileged features:
\begin{align*}
\xi_i(\cvw, \cb) \bydef \inner{\cvw, \cvz_i} + \cb ,
\end{align*}
where $(\cvw, \cb)$ are the additional parameters to be learned.
The following optimization problem defines the SVM$+$ algorithm.
\begin{equation}
\label{eq:svmp_primal}
\begin{aligned}
\min_{\vw, b, \cvw, \cb} & \quad \frac{1}{2}
    (\norms{\vw} + \gamma \norms{\cvw})
    + C \sum_{i=1}^{n} \xi_i(\cvw, \cb) \\
\st & \quad y_i (\inner{\vw, \vz_i} + b) \geq 1 - \xi_i(\cvw, \cb) \\
    & \quad \xi_i(\cvw, \cb) \geq 0
\end{aligned}
\end{equation}
Note that there are two hyper-parameters, $\gamma$ and $C$, that control the
trade-off between the three terms of the objective, where the second term
limits the capacity of the set of correcting functions $\xi_i(\cvw, \cb)$.

\subsection{The WSVM optimization problem}
\label{sec:background_svmw}

The weighted support vector machine (WSVM) is a well-known generalization
of the standard SVM. Each instance $(\vx_i,y_i)$ is assigned an importance
weight $c_i \in \Rb_+$ and in place of the standard empirical risk estimator
$\hat{L}(f) \bydef n^{-1} \sum_{i=1}^{n} \loss(y_i f(\vx_i))$
its weighted version is employed:
\begin{align*}
\hat{L}_{w}(f) \bydef \sum_{i=1}^{n} c_i \loss(y_i f(\vx_i)) .
\end{align*}

The WSVM optimization problem is given below.
\begin{equation}
\label{eq:svmw_primal}
\begin{aligned}
\min_{\vw, b, \vxi} &
    \quad \frac{1}{2} \norms{\vw} + \sum_{i=1}^{n} c_i \xi_i \\
\st & \quad y_i (\inner{\vw, \vz_i} + b) \geq 1 - \xi_i , \; \xi_i \geq 0
\end{aligned}
\end{equation}

At first glance, it may appear that the two generalizations of the SVM
are unrelated. As will become clear in the following, however,
there is a relation between the two and the solution space of WSVMs
includes SVM$+$ solutions. 
This is not very surprising as soon as one realizes that re-weighting
allows to alter the loss function to a large extent and, in particular,
one can mimic the effect of privileged features.
The close relationship can already be seen when comparing the corresponding
dual problems.

\subsection{The dual optimization problems}
\label{sec:background_duals}

Let $\valpha$ and $\vbeta$ be the Lagrange dual variables of the SVM$+$
or the WSVM problem corresponding respectively to the first and
the second inequality constraints
\citep{schoelkopf:2002:learning, vapnik:2009:luhi}.
Define $\cvalpha \bydef \valpha + \vbeta - \vc$, where for the SVM$+$ we set
$\vc = C\ones$, and note that $\vbeta$ can be eliminated leading to the
constraint $\alpha_i \leq c_i + \calpha_i$. Let
\begin{align*}
F(\valpha) &\bydef
\frac{1}{2} \valpha^{\top} \mY\mK\mY \valpha - \ones^{\top} \valpha , &
\cF(\cvalpha) &\bydef \frac{1}{2} \cvalpha^{\top} \cmK \cvalpha .
\end{align*}
It is not hard to see that the following optimization problem
is equivalent to the dual of the SVM$+$ problem (\ref{eq:svmp_primal}).
\begin{equation}
\label{eq:svmp_dual}
\begin{aligned}
\min_{\valpha, \cvalpha} &
    \quad F(\valpha) + \frac{1}{\gamma} \cF(\cvalpha) \\
\st &
    \quad \vy^{\top} \valpha = 0, \;
    \ones^{\top} \cvalpha = 0, \;
    0 \leq \alpha_i \leq C + \calpha_i
\end{aligned}
\end{equation}
Likewise, the problem below is equivalent to the dual of the WSVM
problem (\ref{eq:svmw_primal}).
\begin{equation}
\label{eq:svmw_dual}
\begin{aligned}
\min_{\valpha} &
    \quad F(\valpha) \\
\st &
    \quad \vy^{\top} \valpha = 0, \;
    0 \leq \alpha_i \leq c_i
\end{aligned}
\end{equation}

Note that the constraint $\alpha_i \leq c_i + \calpha_i$ is the
crucial part of the SVM$+$ problem as it introduces the coupling
between the decision space $\Xc$ and the correcting space $\cXc$.
Recall from the representer theorem \citep{schoelkopf:2001:representer}
that an SVM solution has the form
$f = \sum_{i=1}^{n} \alpha_i y_i k(\vx_i, \cdot)$.
Correcting features thus control the maximum influence a data
point $(\vx_i,y_i)$ can have on the resulting classifier,
just like the weights in WSVMs.

\section{Uniqueness results}
\label{sec:uniqueness}

The connection between SVM$+$ and WSVM explored in
Section~\ref{sec:svmp_svmw_relation} relies on the analysis
of uniqueness of their solutions.
Effectively, the statements can only be made with respect to
the classes of equivalent solutions and equivalent weights,
hence, it is imperative to first obtain a better understanding
of different sources of non-uniqueness in the aforementioned problems.

In this section, we show that every non-trivial SVM$+$ solution is unique,
unlike WSVM solutions that may have a non-unique offset $b$. Furthermore,
we describe a set of equivalent weights that yield the same WSVM solutions.
The latter will be used to prove equivalence between the SVM$+$ and the
WSVM algorithms under additional constraints.

\subsection{Uniqueness of WSVM and SVM+ solutions}
\label{sec:unique_solution}

We begin with a known result due to \citet{burges:1999:unique}
that characterizes uniqueness of the weighted SVM solution.
Essentially, it states that if there is an equilibrium
between instance weights of support vectors, then the separating
hyperplane can be shifted within a certain range without altering
the total cost in the WSVM problem. In that case, a WSVM solver has
to choose a value for the offset using some heuristic, e.g.,
it can choose the middle point in the allowed range of $b$.

\begin{theorem}
\label{thm:svmw_unique}
The solution to the problem (\ref{eq:svmw_primal}) is unique in $\vw$.
It is not unique in $b$ and $\vxi$ iff one of the following
two conditions holds:
\begin{align*}
\sum_{i \in \Ic_{-} \cap \Ic_{0}} c_i = \sum_{i \in \Ic_{+} \cap \Ic_{1}} c_i ,
\quad
\sum_{i \in \Ic_{+} \cap \Ic_{0}} c_i = \sum_{i \in \Ic_{-} \cap \Ic_{1}} c_i .
\end{align*}
\end{theorem}
Note that in practice it may happen that one of the two conditions holds and
the WSVM problem (\ref{eq:svmw_primal}) does not have a unique solution.
This is not the case for the SVM$+$ as shown next.

\begin{theorem}
\label{thm:svmp_unique}
The solution to the problem (\ref{eq:svmp_primal}) is unique in $(\vw,\cvw,\cb)$
for any $C > 0$, $\gamma > 0$. If there is a support vector, then $b$ is unique
as well, otherwise:
\begin{align*}
\max_{i \in \Ic_{+}} (1 - \inner{\cvw,\cvz_i} - \cb) \leq b \leq 
\min_{i \in \Ic_{-}} (\inner{\cvw,\cvz_i} + \cb - 1) .
\end{align*}
\end{theorem}
This result is interesting on its own, since it shows that the SVM$+$
is formulated in a way that privileged features always give
enough information to choose \emph{the} unique solution
(if there are no support vectors, then the constant classifier
can be given by $b = \pm 1$ depending on the class balance).

Results concerning uniqueness of dual solutions are more technical
and are moved to the Appendix.

\subsection{Equivalent weights}
\label{sec:equivalent_weights}

Apart from the conditions discussed in the previous section,
another source of non-uniqueness is that any given WSVM solution
corresponds, in general, to multiple weight vectors $\vc$.
In this section, we give a characterization of all such vectors.

\begin{definition}
\label{def:equiv_weights}
A family of equivalent weights $\Wc$ is defined for a given
WSVM solution as follows:
\begin{align*}
\Wc \bydef \{ & \vmu + \vnu \; | \;
    \vmu \in \Uc, \; \vnu \in \Vc \} , \\
\Uc \bydef \{ & \vmu \in \Rb_{+}^n \; | \;
    \textstyle\sum_{i} \mu_i y_i \vz_i = \vw^{\star}, \;
    \textstyle\sum_{i} \mu_i y_i = 0, \\
&
    \textstyle\sum_{i} \mu_i = \textstyle\sum_{i} \alpha^{\star}_i , \;
    \mu_i (\xi^{\star}_i - h_i) = 0 \; \forall i
    \}, \\
\Vc \bydef \{ & \vnu \in \Rb_{+}^n \; | \;
    \nu_i \xi^{\star}_i = 0 \; \forall i
    \},
\end{align*}
where $h_i \bydef [1 - y_i (\inner{\vw^{\star}, \vz_i} + b^{\star})]_{+} $
is the hinge loss at a point $i = 1, \ldots, n$.
\end{definition}

The following simple statement shows that the set $\Wc$
defined above contains \emph{all} weights that correspond to
a given WSVM solution.

\begin{proposition}
\label{prop:equiv_weights}
Let $(\vw^{\star},b^{\star},\vxi^{\star},\valpha^{\star}, \vbeta^{\star})$ be a
primal-dual optimal point for the WSVM problem (\ref{eq:svmw_primal}).
The point $(\vw^{\star},b^{\star},\vxi^{\star})$
is primal optimal for any weight vector $\vc \in \Wc$,
and all such weights are contained in $\Wc$.
\end{proposition}

\begin{corollary}
\label{cor:equiv_weights}
There always exists a weight vector $\vc' \in \Wc$ such that
$\vc' = \valpha' = \valpha^\star$ and $\vbeta' = \zeros$.
\end{corollary}

It is not surprising that \emph{a posteriori} all weight could be
concentrated on support vectors as suggested by
Corollary~\ref{cor:equiv_weights}.
As will become clear in the following, this is close to
what the SVM$+$ algorithm is constrained to do.

\section{Relation between SVM+ and WSVM}
\label{sec:svmp_svmw_relation}

In this section, we present our main theoretical result on the conditions
under which the SVM$+$ and the WSVM are equivalent.
Section~\ref{sec:svmp_to_svmw} shows that it is always possible
to construct weights from an SVM$+$ solution such that the WSVM
will have the same solution.
Section~\ref{sec:svmw_to_svmp} discusses when it is possible
to go in the opposite direction and reveals a fundamental
constraint of the SVM$+$ algorithm.
Finally, Section~\ref{sec:svmp_eq_svmw} states
the necessary and sufficient condition for their equivalence.
Furthermore, we present a counterexample violating
that condition in Section~\ref{sec:example_svmw_not_svmp}
and discuss SVM$-$ in Section~\ref{sec:svmminus}.

\subsection{SVM+ solutions are also WSVM solutions}
\label{sec:svmp_to_svmw}

The following theorem shows that any SVM$+$ solution is also a solution to the
WSVM problem with appropriately chosen weights and such a choice of weights
can always be given by the SVM$+$ dual variables.

\begin{theorem}
\label{thm:svmw_from_svmp}
Let
$(\vw^{\star},b^{\star},\cvw^{\star},\cb^{\star},
\valpha^{\star},\vbeta^{\star})$
be a primal-dual optimal point for the SVM$+$ problem.
There exists a choice of weights $\vc$, namely
$\vc = \valpha^{\star} + \vbeta^{\star}$, and $\vxi^{\star}$ such that
$(\vw^{\star},b^{\star},\vxi^{\star},\valpha^{\star},\vbeta^{\star})$ is a
primal-dual optimal point for the WSVM problem.
\end{theorem}

Note that a direct corollary of this result is that, just like a good
choice of privileged features leads to improved predictive performance
of the SVM$+$ \citep{pechyony:2010:tlpi},
a good choice of weights leads to improved performance of the WSVM.
This claim is verified empirically in the experimental
Section~\ref{sec:experiments} when weights are learned in an
idealized setting, which is close to the Oracle SVM setting of
\citet{vapnik:2009:lupi}.

\begin{figure}[ht]
\centering
\centerline{\includegraphics[width=\columnwidth]{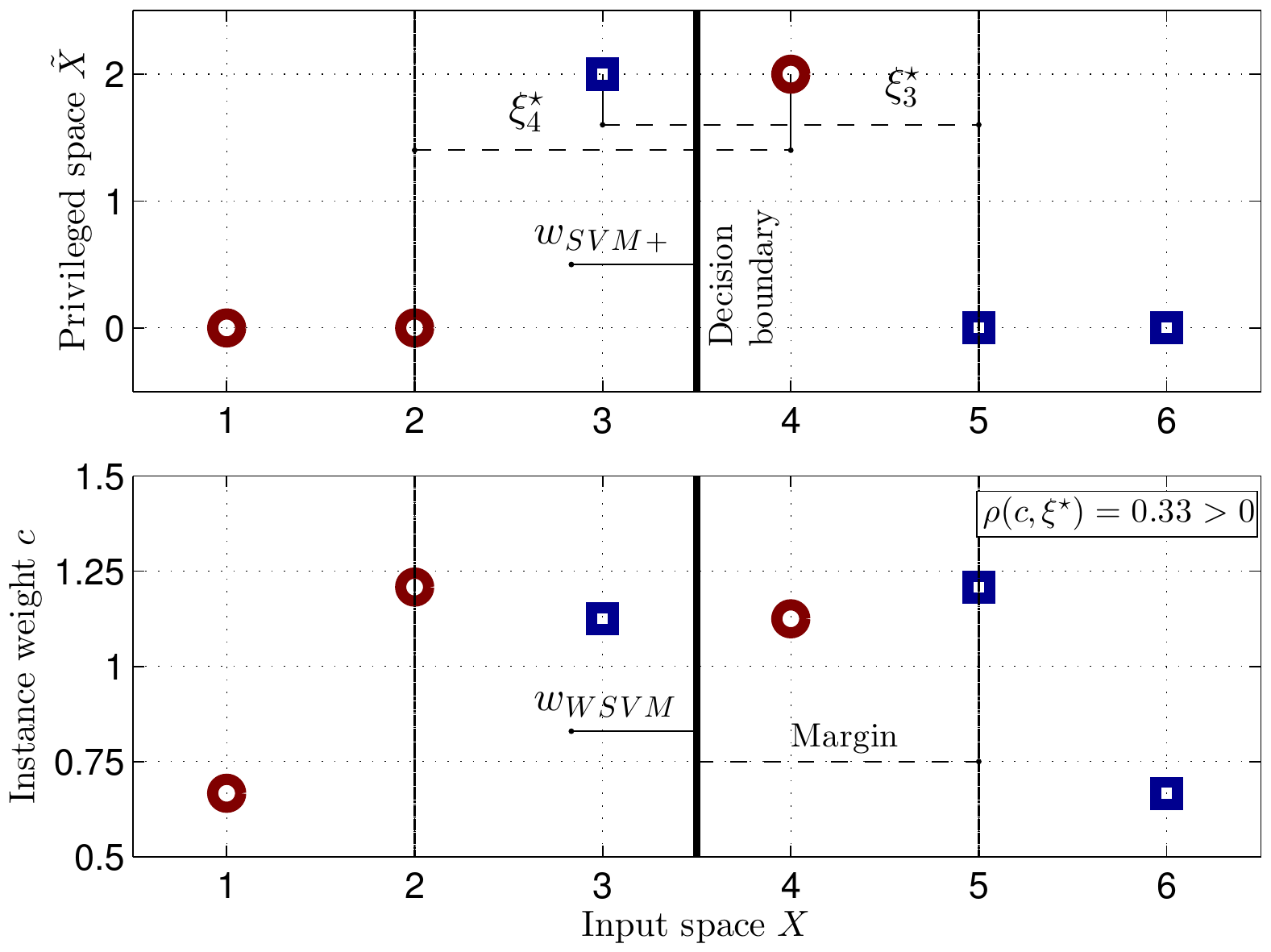}}
\caption{An example of equivalence between SVM$+$ (top)
and WSVM (bottom).
The privileged features coincide with the optimal slack variables
$\xi_i^{\star}$, as motivated by the LUPI paradigm,
and instance weights $c_i$
are given by the sum of SVM$+$ dual variables
(Theorem~\ref{thm:svmw_from_svmp}).
Note that whenever a WSVM solution is constructed from an SVM$+$
solution, as in this case,
the weighted average loss is greater than the non-weighted one,
i.e., $\rho(\vc, \vxi^{\star}) \geq 0$ (Theorem~\ref{thm:svmw_to_svmp_iff}).}
\label{fig:equiv_example}
\end{figure}

Figure~\ref{fig:equiv_example} shows a toy example where an SVM$+$
solution is used to compute weights $\vc = \valpha^{\star} + \vbeta^{\star}$
that force the WSVM to find exactly the same solution.
Note that the outliers (points 3 and 4) receive relatively high
weight, so that the weighted average loss is greater than
the non-weighted one.
See Section~\ref{sec:svmp_eq_svmw} for further details.

\subsection{Which WSVM solutions are SVM+ solutions?}
\label{sec:svmw_to_svmp}

We now consider the opposite direction and characterize the SVM$+$
solutions in terms of the induced instance weights.
The following Lemma~\ref{lem:svmp_necessary} highlights the bias
of the SVM$+$ algorithm as it establishes that every solution must
satisfy a certain relation between the dual variables
(respectively the weights) and the loss on the training sample.
This is the key to showing that the SVM$+$ and the WSVM algorithms
are not equivalent, and that the latter is strictly more generic
as it does not impose that additional constraint.

\begin{lemma}
\label{lem:svmp_necessary}
Assume any given $C > 0$, $\gamma \geq 0$ and let
$(\vw^{\star},b^{\star},\cvw^{\star},\cb^{\star},
\valpha^{\star},\vbeta^{\star} )$ be a primal-dual optimal point for
the SVM$+$ problem (\ref{eq:svmp_primal}), then the following holds:
\begin{align}
\label{eq:svmp_necessary}
\frac{ \sum_{i=1}^{n} (\alpha^{\star}_i + \beta^{\star}_i) h_i }{ \sum_{i=1}^{n}
(\alpha^{\star}_i + \beta^{\star}_i) } \geq \frac{1}{n} \sum_{i=1}^{n} h_i ,
\end{align}
where $h_i \bydef [1 - y_i (\inner{\vw^{\star}, \vz_i} + b^{\star})]_{+} $
is the hinge loss at a point $i = 1, \ldots, n$.
If $\gamma = 0$, then (\ref{eq:svmp_necessary}) is satisfied with equality.
\end{lemma}

Taking into account that the corresponding weights in the WSVM are given by
the sum of the SVM$+$ dual variables, the above inequality can be
re-written in a more compact form.

\begin{corollary}[The Necessary Condition]
\label{cor:svmp_from_svmw_necessary}
Assume the setting of Theorem~\ref{thm:svmw_from_svmp}, then $\xi^{\star}_i =
h_i$ and
\begin{align*}
\inner{\vc - \bar{c} \ones, \vxi^{\star}} & \geq 0 , &
    \text{where } \bar{c} & \bydef n^{-1} \textstyle\sum_{i=1}^{n} c_i .
\end{align*}
\end{corollary}
\begin{proof}
Follows from Theorem~\ref{thm:svmw_from_svmp} and
Lemma~\ref{lem:svmp_necessary}.
\end{proof}

Note that this result suggests a simple way to interpret the effect
of privileged features -- they impose a re-weighting of the input
training data. Moreover, at the end of training more emphasis
will be on points with positive loss and less on easy points,
in particular, the non-support vectors may end up with zero weight.

\subsection{SVM+ and WSVM equivalence}
\label{sec:svmp_eq_svmw}

We now state the main result of this paper which gives the
necessary and sufficient condition for the equivalence between
the SVM$+$ and the WSVM.

\begin{theorem}
\label{thm:svmw_to_svmp_iff}
Let
$(\vw^{\star},b^{\star},\vxi^{\star},\valpha_{0}^{\star},\vbeta_{0}^{\star})$
be a primal-dual optimal point for the WSVM problem
with instance weights $\vc_0 \in \Rb_{+}^{n}$, not all zero.
There exists a choice of $C$, $\gamma$,
and correcting features $\{\cvx_i\}_{i=1}^{n}$
such that $(\vw^{\star},b^{\star})$ is optimal for the SVM$+$ problem iff:
\begin{align}
\label{eq:svmw_to_svmp_iff}
\exists \: \vc \in \Wc \; : \;
\rho(\vc, \vxi^{\star}) \bydef
    \inner{\vc - \bar{c} \ones, \vxi^{\star}} \geq 0 ,
\end{align}
where $\bar{c} \bydef n^{-1} \sum_{i=1}^{n} c_i$.
If $\rho(\vc, \vxi^{\star}) \geq 0$, one such possible choice is as follows:
\begin{align}
\label{eq:svmw_to_svmp_feat}
C &= \bar{c} , & 
\gamma &= \rho(\vc, \vxi^{\star}) , &
\cx_i &= \xi^{\star}_{i} - \cb^{\star} , \; \forall i
\end{align}
moreover, the optimal $\cw^{\star}$ and $\cb^{\star}$ in that case are:
\begin{align}
\label{eq:svmw_to_svmp_copt}
\cw^{\star} &= 1, & \cb^{\star} &= \inner{\vc, \vxi^{\star}} / \inner{\vc,
\ones} .
\end{align}
\end{theorem}

Let us make a few remarks. First, condition (\ref{eq:svmw_to_svmp_iff})
can be rewritten in terms of averages as
\begin{align}
\label{eq:svmw_to_svmp_iff_w}
\sum_{i=1}^{n} \omega_i \xi^{\star}_i & \geq
    \frac{1}{n} \sum_{i=1}^{n} \xi^{\star}_i ,
\end{align}
where $\omega_i \bydef c_i / \sum_{i=1}^{n} c_i$ is the normalized weight.
Hence, any SVM$+$ solution has an equivalent WSVM setting
that puts \emph{more weight on hard examples}, i.e.,
the points with higher loss.

Further, it is clear from Definition~\ref{def:equiv_weights} that
the weight of points with $y_i f(\vx_i) > 1$ can be changed arbitrarily
without altering the $f$ since in that case $\xi^{\star}_i = 0$,
$\alpha^{\star}_i = 0$ and $\beta^{\star}_i = c_i$,
i.e., these points are not support vectors and they
have no influence on the final classifier. Hence, their weight --
the upper bound on the influence -- does not matter.

This reasoning leads us to a condition that is much easier to check
in practice than the one in Theorem~\ref{thm:svmw_to_svmp_iff}.
Note that condition (\ref{eq:svmw_to_svmp_iff}) involves the
set of equivalent weights and it \emph{is} possible to check it directly
using the definition of $\Wc$ as will be discussed below.
However, if the kernel matrix is non-singular,
as is often the case with the Gaussian kernel,
then one can simply take $\vc = \valpha^{\star}$
and check (\ref{eq:svmw_to_svmp_iff}) for that particular
weight vector \emph{only}.

\begin{proposition}
\label{prop:svmw_to_svmp_iff_easy}
Let
$(\vw^{\star},b^{\star},\vxi^{\star},\valpha^{\star},\vbeta^{\star})$
be a primal-dual optimal point for the WSVM problem
with instance weights $\vc \in \Rb_{+}^{n}$, not all zero. If
\begin{align*}
\Null(\mY\mK\mY) \cap \ones^{\perp} \cap \vy^{\perp} = \{\zeros\} ,
\end{align*}
then there exists a choice of $C$, $\gamma$, and $\{\cvx_i\}_{i=1}^{n}$
such that $(\vw^{\star},b^{\star})$ is optimal for the SVM$+$ problem iff:
\begin{align}
\label{eq:svmw_to_svmp_iff_easy}
\rho(\valpha^{\star}, \vxi^{\star}) =
    {\vxi^{\star}}^{\top} \big( \Id - \tfrac{1}{n}\ones\ones^{\top} \big)
    \valpha^{\star} \geq 0 .
\end{align}
\end{proposition}

Intuitively, the SVM$+$ algorithm maximizes the margin $2 \norm{\vw}^{-1}$
by minimizing $F(\valpha)$, as in the standard SVM, and also
gradually shifts focus to hard examples by minimizing $\cF(\cvalpha)$.
As long as there are sufficiently many points
on the ``right'' side of the margin,
(\ref{eq:svmw_to_svmp_iff_w}) can be achieved by reducing the weight
of such non-support vectors, and so the SVM$+$ solution space is
as rich as that of the WSVM.
In general, however, (\ref{eq:svmw_to_svmp_iff_w}) may
not be attainable without altering the $f$ as demonstrated by the counter
example below.

\subsection{WSVM solution not found by SVM+}
\label{sec:example_svmw_not_svmp}

We now consider the case when misclassified training points have
low weight, i.e., $\rho(\vc, \vxi^{\star}) < 0$, and give an example
where SVM$+$ fails to find the corresponding WSVM solution.

Consider the training sample below (Figure~\ref{fig:counter_example}):
\begin{align*}
S &= \{(1, +1), \: (2, -1), \: (3, +1)\}, & \vc & = (4, 6, 2)^{\top} .
\end{align*}
The corresponding primal-dual optimal point is
\begin{align*}
w^{\star} &= -2, &
    \vxi^{\star} &= (0, 0, 4)^{\top}, &
    \valpha^{\star} &= (4, 6, 2)^{\top}, \\
b^{\star} &= 3, &&&
    \vbeta^{\star} &= (0, 0, 0)^{\top}.
\end{align*}

Since $\rho(\vc, \vxi^{\star}) = - \frac{2}{3} < 0$, this solution does not
correspond to any of the SVM$+$ solutions (Lemma~\ref{lem:svmp_necessary}).
Note that one can easily verify that
$\Null(\mY\mK\mY) \cap \ones^{\perp} \cap \vy^{\perp}$
contains only $\zeros$, hence, Proposition~\ref{prop:svmw_to_svmp_iff_easy}
already completes the claim. Similarly, one can show using
Definition~\ref{def:equiv_weights} that $\Uc = \{\valpha^{\star}\}$
and that other equivalent weights can only increase the weight
of points 1 and 2, which would only decrease $\rho(\vc, \vxi^{\star})$.
Therefore, there is no $\vc' \in \Wc$ for which
$\rho(\vc',\vxi^{\star}) \geq 0$ and,
by Theorem~\ref{thm:svmw_to_svmp_iff},
there is no correcting space that would make
$(w^{\star}, b^{\star}) = (-2, 3)$ an SVM$+$ solution.

\begin{figure}[ht]
\centering
\centerline{\includegraphics[width=\columnwidth]{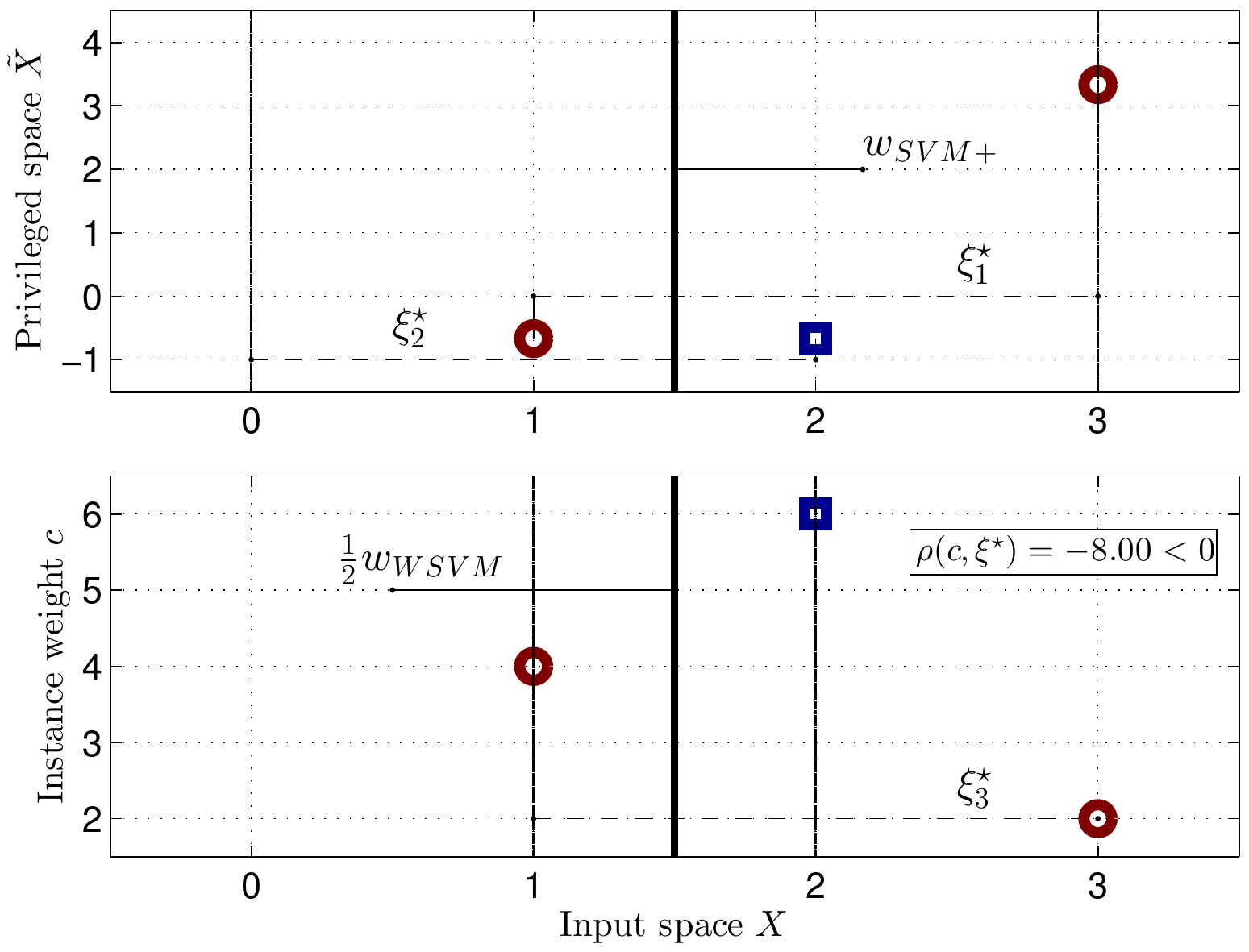}}
\caption{An example of a WSVM solution (bottom)
that cannot be found by SVM$+$ (top).
The instance weights $c_i$ are chosen in a way to avoid a zero norm
constant classifier ($f = +1$).
The resulting weighted average loss is less than the non-weighted one,
hence the SVM$+$ cannot find this solution.
Computing the privileged features as in (\ref{eq:svmw_to_svmp_feat})
leads to an SVM$+$ solution with the opposite prediction
and a higher value of the weighted average loss.}
\label{fig:counter_example}
\end{figure}

Figure~\ref{fig:counter_example} shows the learned WSVM and SVM$+$ models,
where we used
$\cx_i = \xi^{\star}_i - \inner{\vc, \vxi^{\star}} / \inner{\vc,\ones}$,
$C = \bar{c}$, $\gamma = 1$.
A different choice of $C$ and $\gamma$ can make SVM$+$ return
a constant classifier, which is the solution of the standard SVM,
but there is no setting that would make it return
$(w^{\star}, b^{\star}) = (-2, 3)$.

Note that in this example an even stronger result can be shown:
SVM$+$ cannot reproduce the same \emph{type} of dichotomy, i.e.,
even if we allowed it to return a line with \emph{any} negative slope
going through the same point, the SVM$+$ would still fail.
This shows that there are settings where WSVM performs significantly
better than SVM$+$ due to a fundamental constraint of the latter.

\subsection{Is there an \texorpdfstring{SVM$-$}{SVM-}?}
\label{sec:svmminus}

We have seen that the SVM$+$ has a more constrained solution space
than the weighted SVM.
Lemma~\ref{lem:svmp_necessary} gives the exact characterization of
that constraint in terms of the relation between the SVM$+$ dual
variables and the incurred loss on the training sample.
The WSVM solution space can thus be partitioned into solutions
that can be found by SVM$+$ and the rest.
We are now interested if there is a modification to the SVM$+$
algorithm that would yield solutions from that second part.

Theorem~\ref{thm:svmw_to_svmp_iff} suggests that
$\gamma = \rho(\vc, \vxi^{\star}) \geq 0$, so, intuitively,
if we now require $\rho(\vc, \vxi^{\star}) < 0$, the corresponding
$\gamma$ has to be with a minus:
\begin{equation}
\label{eq:svmm_primal}
\begin{aligned}
\min_{\vw, b, \cvw, \cb} & \quad
    \frac{1}{2} (\norms{\vw} - \gamma \norms{\cvw})
    + C \sum_{i=1}^{n} \xi_i(\cvw, \cb) \\
\st & \quad y_i (\inner{\vw, \vz_i} + b) \geq 1 - \xi_i(\cvw, \cb) \\
    & \quad \xi_i(\cvw, \cb) \geq 0
\end{aligned}
\end{equation}

This problem is clearly non-convex as the objective is now
a difference of convex functions. If there was a finite (local)
minimizer $(\vw^{\star},b^{\star},\cvw^{\star},\cb^{\star})$,
the KKT conditions would still hold
\citep[Theorem~2.3.8]{borwein:2000:convex}
for a Lagrange multiplier vector $(\valpha^{\star},\vbeta^{\star})$,
and one could show a result similar to Lemma~\ref{lem:svmp_necessary},
but with the reverse inequality.

Unfortunately, however, the problem (\ref{eq:svmm_primal})
is unbounded below, which is easy to see: the quadratic term
$\norms{\cvw}$ grows faster than the linear term $\xi_i(\cvw, \cb)$
and the feasible set is unbounded.
This shows that it is not trivial to modify the SVM$+$ algorithm
to obtain solutions from its complement, and it is an open question
if such a modification (with non-degenerate solutions) exists at all.

The phenomenon we observe here is that some of the WSVM solutions
($\rho(\vc, \vxi^{\star}) \geq 0$) can be computed easily
within the LUPI framework, while others ($\rho(\vc, \vxi^{\star}) < 0$)
may be completely out of reach.
What are the implications of this observation
in terms of learning a classifier?

Consider any training sample $S$ of size $n$ for a problem $\Prob$.
Let $f_{\vc,S}$ be a classifier constructed by the WSVM with weights $\vc$,
and let $\vxi^{\star}_{\vc,S}$ be the corresponding loss vector.
The set of all admissible weights $\Rb_{+}^{n}$
is partitioned into two subsets, $\Wc_+$ and $\Wc_-$,
depending on the sign of $\rho(\vc, \vxi^{\star}_{\vc,S})$.
Define the ``best'' weight vectors in each of the two classes as
$\vc_\pm = \argmin_{\vc \in \Wc_\pm} L(f_{\vc,S})$.
If $L(f_{\vc_-,S}) < L(f_{\vc_+,S})$, then the best classifier
corresponds to the weights that are out of reach for the SVM$+$, hence,
there are no privileged features that will yield an SVM$+$ classifier
as good as $f_{\vc_-,S}$.

This reasoning motivated us to consider weight generation schemes
that are unrelated to SVM$+$ and which are discussed next.

\section{How to choose the weights}
\label{sec:choose_weights}

Recall that we are interested in ways of incorporating prior knowledge
about the \emph{training data}.
In the SVM$+$ approach, the role of additional information is played by the
privileged features which are used to estimate the loss on the training sample.
The same effect, as we have established, can be achieved by importance
weighting. Taking into account the vast amount of work on weighted learning,
it seems that re-weighting of misclassification costs
is a very powerful method of incorporating prior knowledge.
We would like to stress, however, that a critical difference to, e.g.,
the cost-sensitive learning is that we are ultimately
interested in minimizing the \emph{non-weighted} expected loss and
\emph{the weights are only used to impose a bias on the learning algorithm}.

We also note that even though the SVM$+$ solutions are contained
within the WSVM solutions, there is \emph{no} implication that
any of the two algorithms is ``better''.
If privileged features are available, then SVM$+$ is a reasonable choice.
On the other hand, if there are no privileged features or if one has
concerns outlined at the end of Section~\ref{sec:svmminus}, then one
may want to consider a more general WSVM with some problem specific
scheme for computing weights.

In the following, we investigate two approaches that make
different assumptions about what is additionally available to the learning
algorithm at training time. The methods operate in a somewhat idealized
setting and are mainly aimed at motivating further research on
how to choose the weights.
They may be thought of as the empirical counterparts of a more
theoretical discussion involving the Oracle SVM in \citep{vapnik:2009:lupi}.
In particular, the weight learning method of
Section~\ref{sec:weight_learning} can be thought of as a way of
extracting additional information about the given training sample
from a validation sample which is used as a reference.

\subsection{Why instance weighting is important?}
\label{sec:instance_weighting_important}

Let us first motivate why instance weighting can be very important
in certain problems.

\begin{figure}[ht]
\centering
\centerline{\includegraphics[width=\columnwidth]{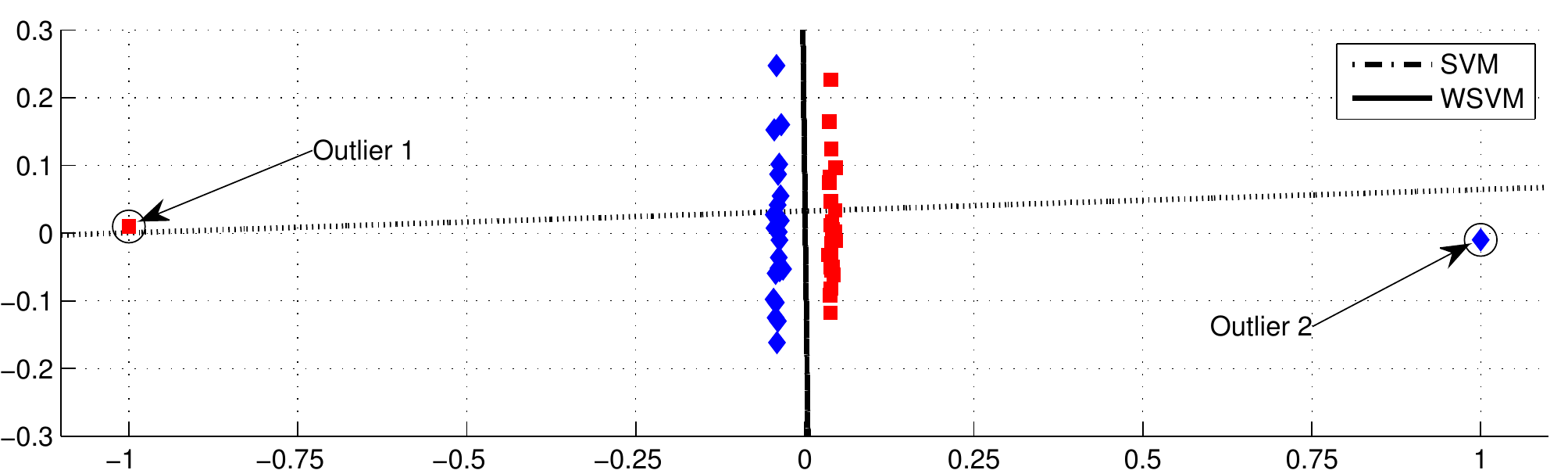}}
\caption{Illustration of the effect of instance weighting
on a toy problem in 2D.
Even though the problem is (almost) linearly separable,
the two outliers in the training set cause the SVM
to have a near chance level performance (horizontal line).
Assigning zero weight to the outliers allows the WSVM
to recover a near optimal solution (vertical line).}
\label{fig:toy_dataset1}
\end{figure}

Consider the toy problem shown in Figure~\ref{fig:toy_dataset1}.
The data comes from two linearly separable blobs, so it is possible
to achieve zero test error on them.
However, the training sample has been
contaminated with two outliers that lie extremely far from the optimal
decision boundary. Since the SVM uses a surrogate loss and not the 0-1 loss,
the cost of a point is higher the further the point is from the
separating hyperplane. Hence, the SVM ``prefers'' to keep the two outliers
close to the decision boundary, which leads to a near chance level
performance on this data set.
Instance weighting, on the other hand, allows one to alter the cost
of each point. In particular, if the two outliers are assigned zero weight,
then the WSVM is able to find a near optimal classifier.

\begin{figure}[ht]
\centering
\centerline{\includegraphics[width=\columnwidth]{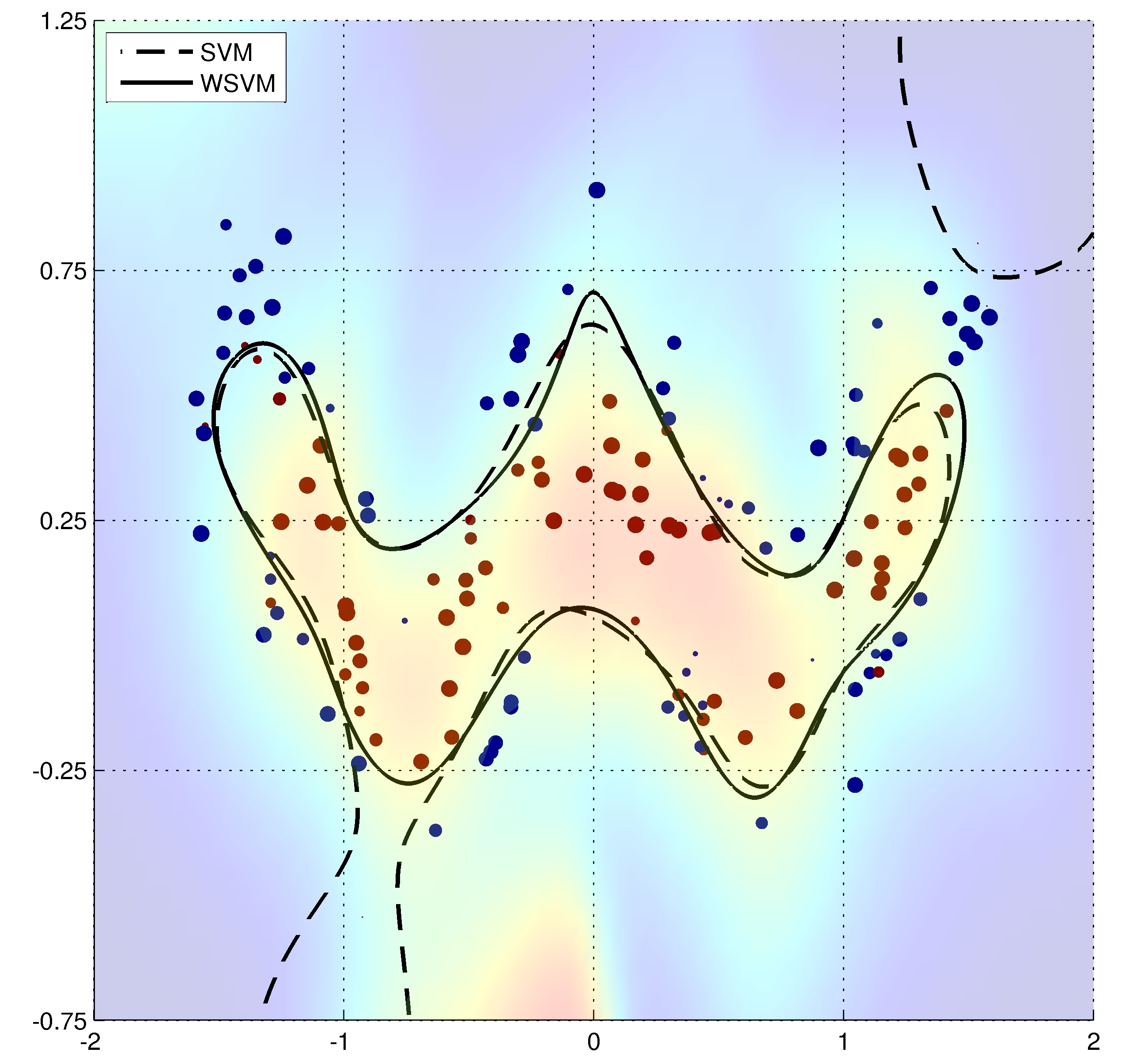}}
\caption{Importance weighting leads to a more stable estimate
of the decision boundary in a non-linear 2D problem.
The size of a data point corresponds to its weight, which is computed
from an estimate of $\Prob(Y=1|X)$ shown in background.
The WSVM (solid line) is less influenced by outliers
than SVM (dashed line) since the outliers are downweighted,
which ultimately results in better predictive performance.}
\label{fig:toy_dataset2}
\end{figure}

The second toy problem shown in Figure~\ref{fig:toy_dataset2} suggests
that an estimate of $\Prob(Y=1|X)$ could be used to compute
instance weights and improve predictive performance even
in the non-linear case, where the aforementioned problem of extreme
outliers is less likely to happen.
As before, the issue evolves around the points that lie either
too close to or even on the wrong side of the true decision boundary.
We used the standard Nadaraya-Watson estimator
(\ref{eq:nadaraya_watson}) to obtain an estimate of the
conditional probability (shown in background), which was then
used to compute instance weights (reflected by the size of points)
using the formula (\ref{eq:wgt:weights}) introduced below.
Note that the outliers are downweighted and have less
influence on the WSVM decision boundary (solid line)
than on the SVM one (dashed line). That leads to better accuracy,
as reported in Section~\ref{sec:exp:toy_data}.

\subsection{Access to an estimate of \texorpdfstring{$\Prob(Y=1|X)$}{P(Y|X)}}
\label{sec:access_cond_prob}

Clearly, having full access to the conditional probability $\Prob(1|X)$ is a
hypothetical scenario since in this case the classification problem is solved.
However, it is interesting to see how this type of information could be used in
construction of good weights.
As the first step, we note that if $\Prob(1|X)$ were available
at least for the \emph{training} points
one could directly compute the conditional
expectation and employ the following estimator
\begin{align*}
L'(f) \bydef & \frac{1}{n} \sum_{i=1}^{n} \big[ \loss(f(X_i)) \Prob(1|X_i) \\
     & + \loss(-f(X_i)) \Prob(-1|X_i) \big] ,
\end{align*}
which is an unbiased estimator of $L(f)$:
\begin{align*}
\Exp L'(f) &= \Exp\left[ \loss(f(X)) \Prob(1|X) + \loss(-f(X)) \Prob(-1|X)
\right] \\
 &= \Exp \Exp\left[ \loss(Y f(X)) | X \right] = \Exp \loss(Yf(X)) = L(f) .
\end{align*}

The property of being biased or not is of asymptotic nature and is arguably of
lesser interest in the small sample regime. Following this line of argument, we
consider a conservative weighted estimator given by:
\begin{align}
\hat{L}_{w}(f) &\bydef
\frac{1}{n} \sum_{i=1}^{n} w(X_i,Y_i) \loss(Y_i f(X_i)) ,
\label{eq:wgt:estimator} \\
w(X_i,Y_i) &\bydef \Prob(Y=Y_i|X=X_i).
\label{eq:wgt:weights}
\end{align}
It is not hard to check that $\hat{L}_{w}(f)$ is biased:
\begin{align*}
& \Exp \hat{L}_{w}(f) \\
& \quad = \Exp \Exp\left[ w(X, Y) \loss(Y f(X)) | X \right] \\
& \quad = \Exp\left[ \loss(f(X)) \Prob(1|X)^2 
    + \loss(-f(X)) \Prob(-1|X)^2 \right] \\
& \quad \leq \Exp\left[ \loss(f(X)) \Prob(1|X) 
    + \loss(-f(X)) \Prob(-1|X) \right] \\
& \quad = \Exp \Exp\left[ \loss(Y f(X)) | X \right] = L(f) .
\end{align*}
More precisely, $\hat{L}_{w}(f)$ is \emph{conservative}
in the sense that the points
far from the decision boundary are upweighted, while the points with
$\Prob(1|X) \approx 0.5$ receive relatively low weight.
This behavior is due to the $p \mapsto p^2$ transform
which is monotonically increasing and is strictly convex on $[0,1]$.
The monotonicity also ensures the following
important property of the obtained estimator when $\loss$ is the 0-1 loss:
\begin{align*}
\argmin_{f} \Exp \hat{L}_{w}(f) = f^* = \argmin_{f} \Exp L(f) ,
\end{align*}
that is, the $\hat{L}_{w}$ is minimized by the Bayes classifier
and \emph{the learning problem is not changed}.

If the bias of $\hat{L}_{w}$ is a concern, one can let the weights decay to one
as the size of the training sample increases. To this end, we consider the
following generalization of the weight function in (\ref{eq:wgt:weights}):
\begin{align}
c_{\tau}(X_i,Y_i) &\bydef w^{\tau}(X_i,Y_i) , \label{eq:wgt:weights_alpha}
\end{align}
where $\tau \in [0, \infty)$ is tuned along with the standard regularization
parameter. Note that SVM is recovered when the weights are given by
$c_0(X_i,Y_i) \equiv 1$.

When $\Prob(Y=1|X)$ is estimated from a training sample, the WSVM with weights
given by (\ref{eq:wgt:weights_alpha}) will mainly serve as a baseline for the
method introduced in the following section. However, it is conceivable that an
estimate of $\Prob(Y=1|X)$ could be available from a different source, e.g.,
from annotations provided by humans. The latter setting is evaluated in
Section~\ref{sec:exp:5_vs_8}.

\subsection{Learning the weights}
\label{sec:weight_learning}

Given a fixed training sample $S$, the weights in a weighted SVM
parametrize the set of hypotheses that the WSVM can choose from. Hence,
they could be learned within the standard framework of risk minimization
with the additional twist that the classifier $f$ depends on
the weights $\vc$ \emph{implicitly}:
\begin{align}
\vc^{\star} &= \argmin_{\vc \in \Rb_{+}^{n}} \Exp \loss(Y f_{\vc}(X)) ,
\label{eq:wlearning:wopt_risk} \\
f_{\vc} &= \argmin_{f \in \Hc} \frac{1}{2} \norms{f} + \sum_{i=1}^{n} c_i
\loss(y_i f(\vx_i)) . \label{eq:wlearning:fopt_risk}
\end{align}

Clearly, the optimization problem (\ref{eq:wlearning:wopt_risk}) cannot be
solved in practice since the underlying probability distribution
is unknown, hence, we replace $L(f)$ in
(\ref{eq:wlearning:wopt_risk}) with an estimator.
The latter, however, has to be different from the estimator
$\hat{L}_{w}$ in (\ref{eq:wlearning:fopt_risk}) to avoid overfitting.
We follow a simple approach and assume that a second sample $S'$
is available at training time.
The problem (\ref{eq:wlearning:wopt_risk}) is thus replaced with
\begin{align*}
\vc^{\star} &= \argmin_{\vc \in \Rb_{+}^{n}} \sum_{i=1}^{N} \loss(y'_i
f_{\vc}(\vx'_i)) .
\end{align*}

This idea follows the method of \cite{chapelle:2002:choosing} who suggested to
tune L2-SVM parameters by minimizing certain estimates of the generalization
error using a gradient descent algorithm. The use of the $L_2$ penalization of
the training errors allows one to additionally assume the hard margin case
which leads to a very specific derivation of the gradient
w.r.t.\ the parameters.
Instead, we proceed with a different approach and use a smooth version
of the hinge loss given below in (\ref{eq:smoothloss}).
Furthermore, we optimize (\ref{eq:wlearning:fopt_risk})
in the primal as suggested by \cite{chapelle:2007:primal}.
The weight learning problem can thus be stated as follows.
\begin{align}
\vc^{\star} &= \argmin_{\vc \in \Rb_{+}^{n}} \sum_{i=1}^{N} \loss(y'_i
[\bar{K}_{i}^{\top} \valpha^{\star}(\vc) + b^{\star}(\vc)]) ,
\label{eq:wlearning:wopt} \\
\begin{bmatrix}
\valpha^{\star} \\
b^{\star}
\end{bmatrix}
    &= \argmin_{\valpha, b} \frac{1}{2} \valpha^{\top}
\mK \valpha + \sum_{i=1}^{n} c_i \loss(y_i [K_{i}^{\top} \valpha + b]) ,
\label{eq:wlearning:fopt}
\end{align}
where $\bar{\mK}$ is the matrix with $\bar{\mK}_{ij} = k(\vx_i,\vx'_j)$,
and $K_i$, $\bar{K}_i$ are the $i^{th}$ columns of $\mK$ and $\bar{\mK}$.

Note that $f$ depends on the weights implicitly via the second optimization
problem and the main challenge in applying the gradient descent is the
computation of $\partial \valpha^{\star} / \partial \vc$
and $\partial b^{\star} / \partial \vc$.
These can be computed via implicit differentiation from the
optimality conditions as shown below.

\begin{theorem}
\label{thm:derivatives}
Let the loss function $\loss$ be convex and twice continuously differentiable
and let the kernel matrix $\mK$ be (strictly) positive definite. Define vectors
$\vu$ and $\vv$ componentwise for $i = 1, \ldots, n$ as
\begin{align*}
u_i &\bydef y_i \loss'(y_i [K_{i}^{\top} \valpha^{\star} + b^{\star}]) , \\
v_i &\bydef c_i \loss''(y_i [K_{i}^{\top} \valpha^{\star} + b^{\star}]) ,
\end{align*}
where $(\valpha^{\star}, b^{\star})$ is a solution of (\ref{eq:wlearning:fopt})
for a given $\vc$.
If $\vv \neq \zeros$, then the solution is unique,
$\valpha^{\star}$ and $b^{\star}$ are continuously differentiable w.r.t.\ $\vc$
and the corresponding gradient can be computed as follows.
\begin{align}
\label{eq:gradient}
\begin{bmatrix}
\frac{\partial \valpha^{\star}}{\partial \vc} \\
\frac{\partial b^{\star}}{\partial \vc}
\end{bmatrix}
= -
\begin{bmatrix}
\Id + \diag(\vv) \mK & \vv \\
\ones^{\top} & 0
\end{bmatrix}^{-1}
\begin{bmatrix}
\diag(\vu) \\
\zeros^{\top}
\end{bmatrix}
\end{align}
\end{theorem}

Note that this result can be directly applied to such popular loss functions as
the squared hinge loss and the logistic loss, and for the latter it
will always hold that $\vv \neq \zeros$ unless all weights are zero.
If $\vv = \zeros$, it can be seen that $\valpha^{\star}$ is
still uniquely defined and is continuously differentiable w.r.t.\ $\vc$
if $b$ is considered fixed. The ``gradient'' in this case is given by
\begin{align}
\label{eq:gradient_vzero}
\partial \valpha^{\star} / \partial \vc &= \diag(\vu), &
\partial b^{\star} / \partial \vc &= \zeros^{\top} .
\end{align}

\begin{figure}[ht]
\centering
\centerline{\includegraphics[width=\columnwidth]{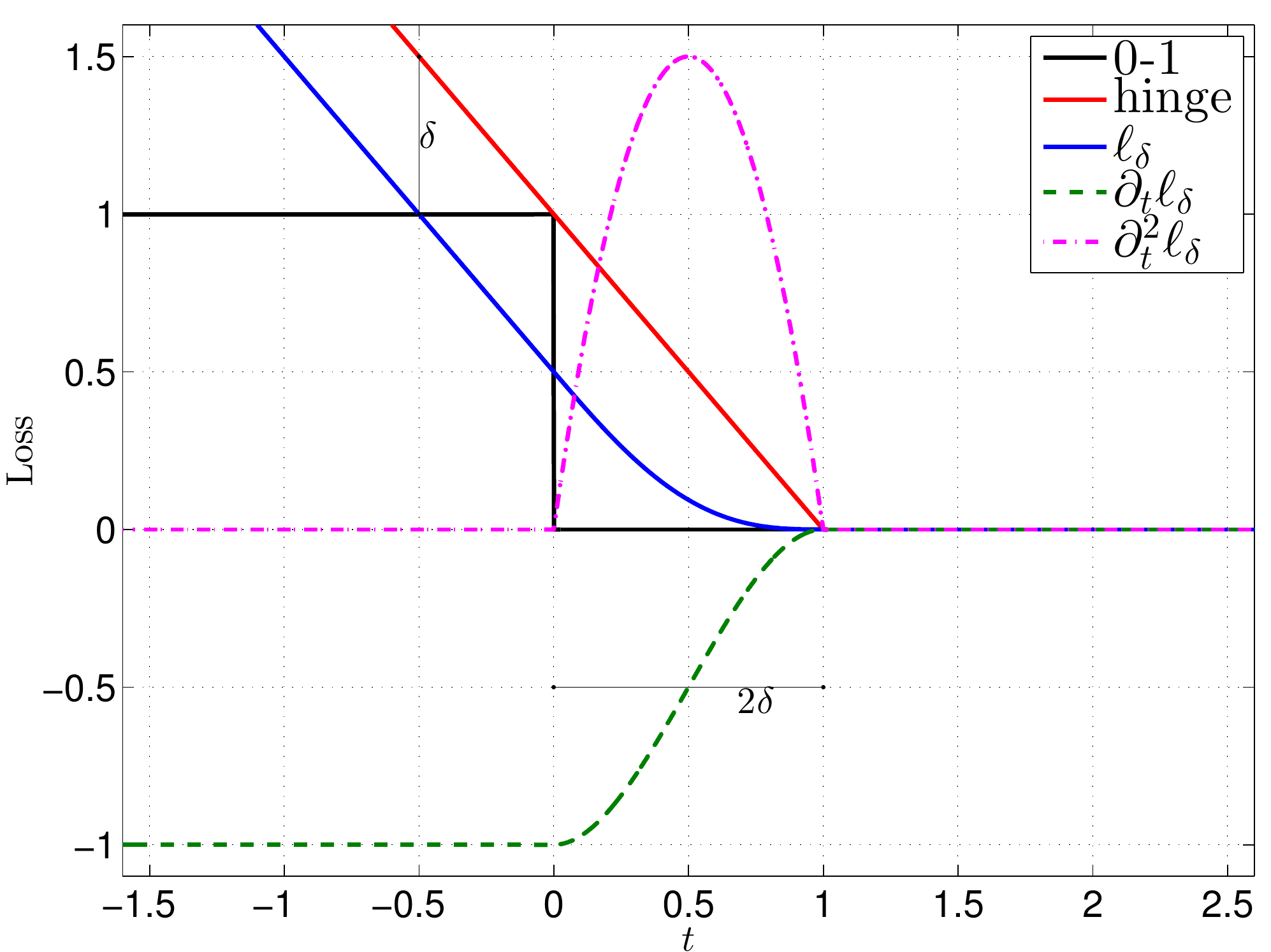}}
\caption{A 0-1 loss, a hinge loss, an approximate hinge loss $\loss_{\delta}$,
and its first two derivatives.}
\label{fig:hinge_losses}
\end{figure}

Ideally, to be consistent with the discussion about the relation between
the SVM$+$ and the WSVM, we would have to consider the hinge loss in the
weight learning problem. However, the hinge loss is not differentiable and
Theorem~\ref{thm:derivatives} does not apply. Instead, we consider a
differentiable approximation of the hinge loss that preserves
certain desirable properties of the latter.
We have chosen the loss function defined as follows
(Figure~\ref{fig:hinge_losses}).
\begin{align}
\label{eq:smoothloss}
\loss_{\delta}(t) \bydef
\begin{cases}
1 - t - \delta  &
    \mbox{if } t \leq 1 - 2 \delta \\
\frac{(1 - t)^3 (t - 1 + 4 \delta)}{16 \delta^3} &
    \mbox{if } 1 - 2 \delta < t < 1 \\
0  & \mbox{if } t \geq 1
\end{cases}
\end{align}
Note that, unlike certain other approximations, this function is twice
continuously differentiable. Like the hinge loss, it does not penalize points
with the margin $t \bydef y_i f(\vx_i) \geq 1$ and it grows linearly for
$t \leq 1 - 2 \delta$.

With the approximate hinge loss $\loss_{\delta}$ defined above, 
the $\vv \neq \zeros$ means that at least one of the data points
has to fall into the strictly convex region of the loss.
Clearly, this presents us with a tradeoff between having
a good approximation of the hinge loss (small $\delta$)
and a higher chance of being able to compute ``correct'' gradients
and thus make substantial progress in the optimization problem
(large $\delta$). We resolve the tradeoff by tuning
$\delta \in [0.01, 1]$ on a validation set.

\section{Experiments}
\label{sec:experiments}

In this section we present empirical evaluation of the algorithms
considered in this paper.
In our experiments, we used the implementation of the WSVM
by \citet{chang:2011:libsvm}
and the code for the SVM$+$ provided by \citet{pechyony:2011:fastopt}.
The weight learning problem was solved using our implementation
of the BFGS algorithm \citep{nocedal:2006:numopt}.
The general experimental setup is similar to that of
\citep{vapnik:2009:lupi}:
parameters are tuned on a validation set, which is not
used for training, and performance is evaluated on a test set.
Training subsets are randomly sampled from a fixed training set,
and results over multiple runs are aggregated showing the mean error rate
as well as the standard deviation.
Depending on the experiment, the validation set is either fixed or subsampled
randomly as well. The Gaussian RBF kernel is used in all of the experiments and
features are rescaled to be in $[0,1]$.
The weights in (\ref{eq:wgt:weights}) are computed from
$\eta(\vx) = 2 \Prob(1|\vx) - 1$,
which is either given directly by human experts or estimated via:
\begin{align}
\label{eq:nadaraya_watson}
\eta(\vx) = \frac{\sum_{i=1}^{n} K_h(\vx - \vx_i) y_i}
{\sum_{i=1}^{n} K_h(\vx - \vx_i)} ,
\end{align}
where $K_h$ is the Gaussian kernel with bandwidth $h$.

Note that in all experiments each algorithm has access to exactly
\emph{the same data}, and the only difference between different splits
is which data is used to construct a classifier (training)
and which is used to tune the hyper-parameters (validation).

\subsection{WSVM replicates SVM+}
\label{sec:exp:svmw_replicates_svmp}

\begin{figure*}[ht]
\centering
\centerline{\hfill
\includegraphics[width=0.33\textwidth]{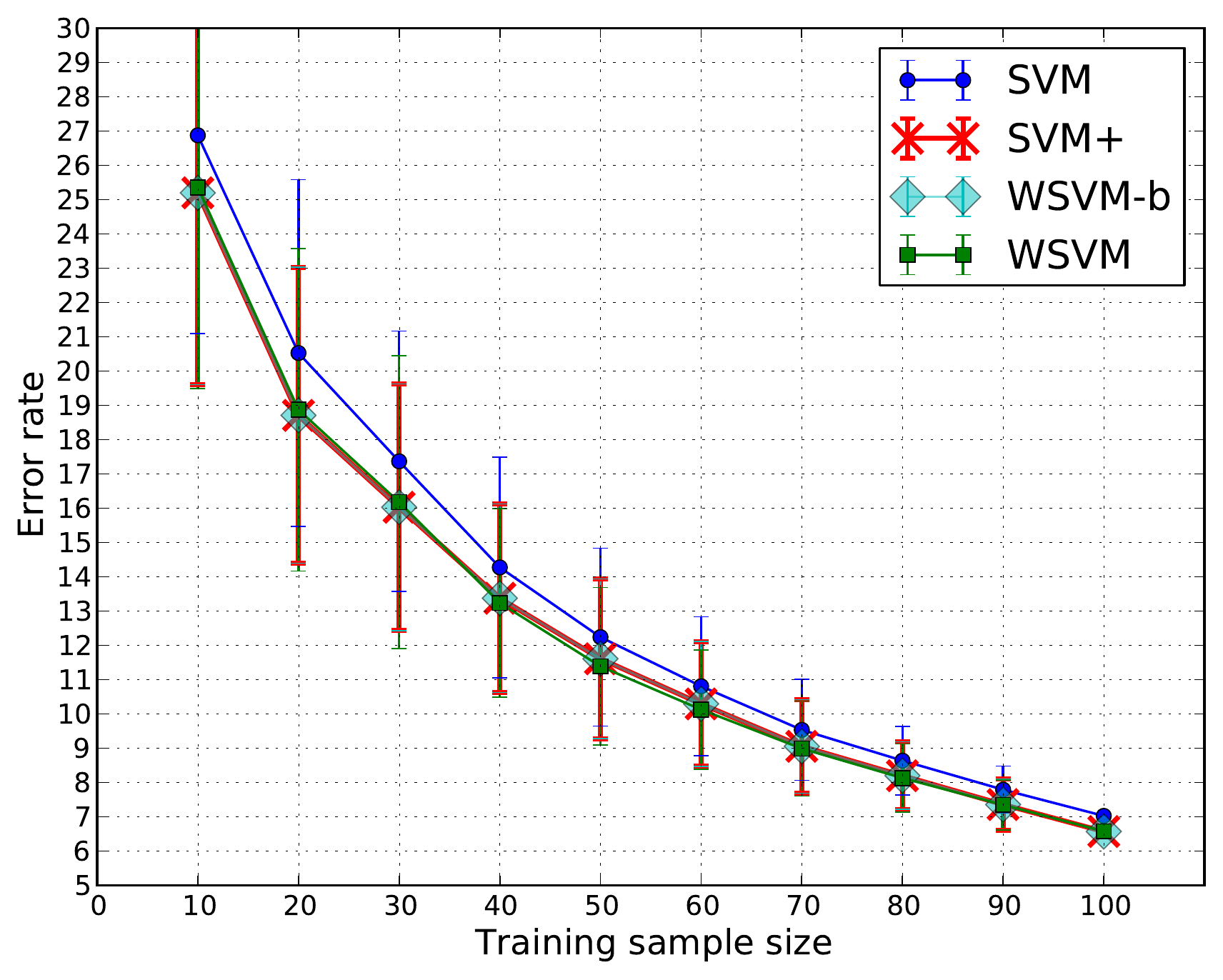}\hfill
\includegraphics[width=0.33\textwidth]{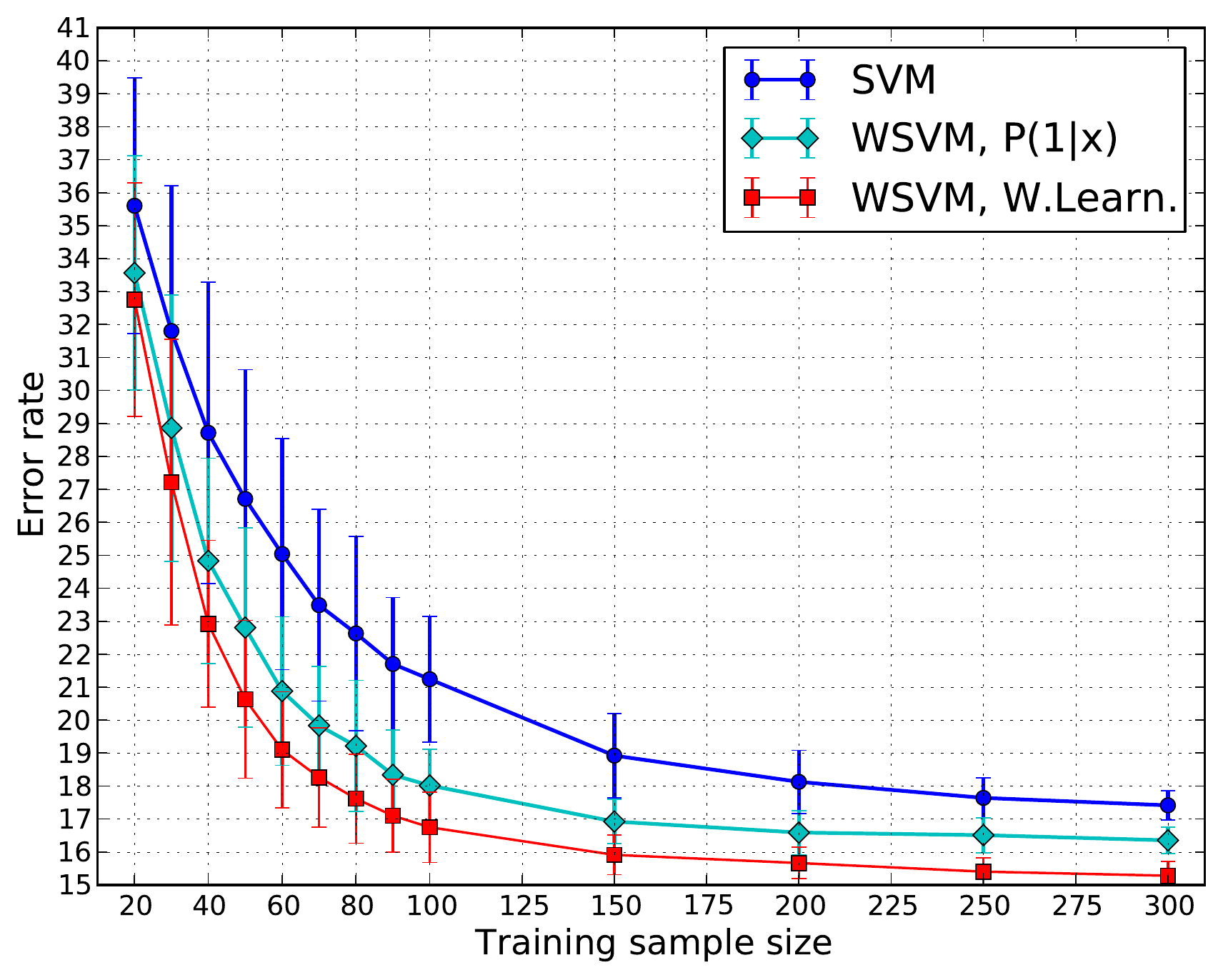}\hfill
\includegraphics[width=0.33\textwidth]{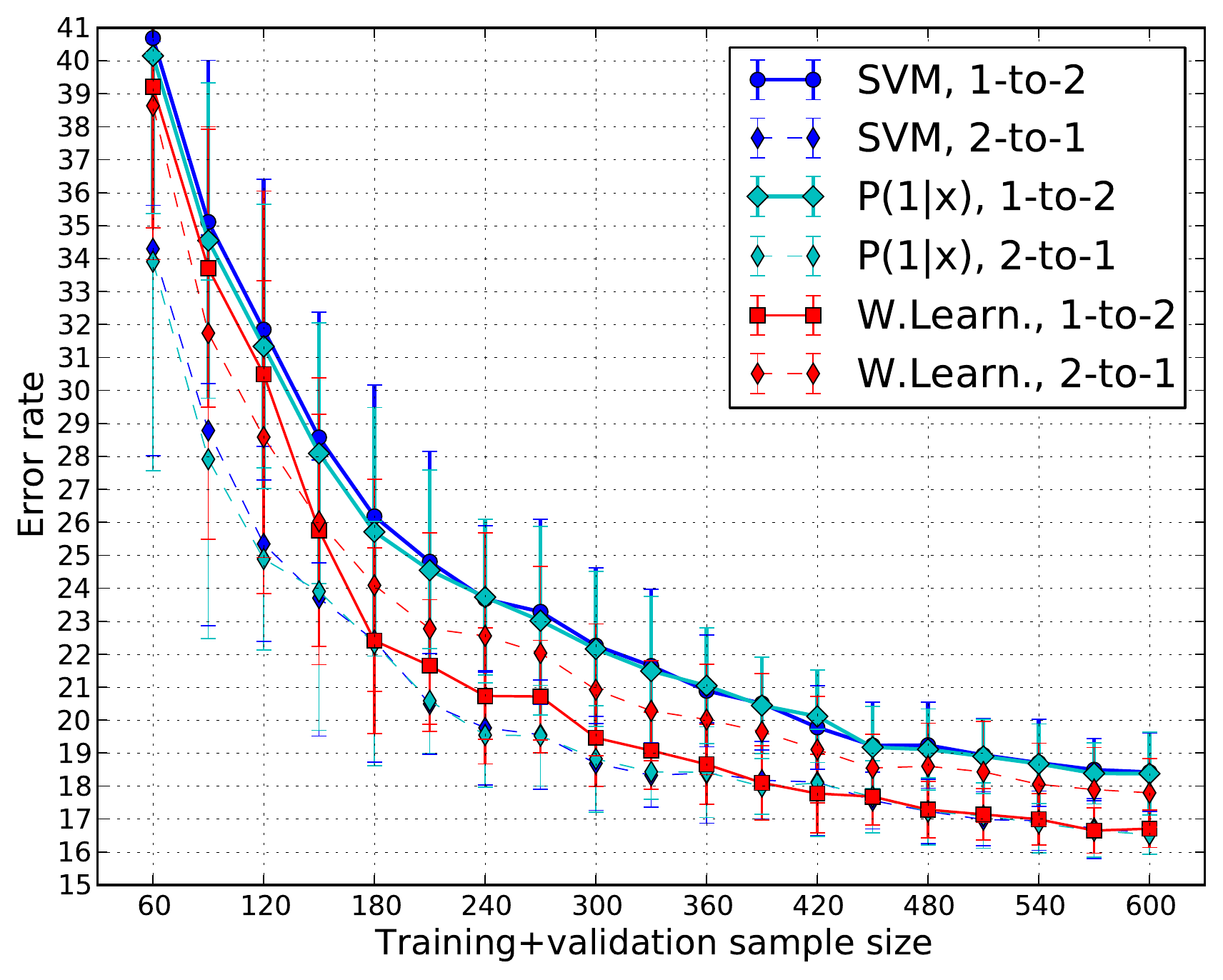}\hfill
}
\caption{SVM, SVM$+$, and WSVM error rates.
Left: Reproduction of the experiment of \cite{vapnik:2009:luhi}.
The SVM$+$ and the WSVM classifiers coincide up to the non-uniqueness of $b$.
Middle: Instance weighting leads to significant performance improvement
when a large validation set is available (toy data).
Right: Similar setting, but the training-to-validation splits are 1-to-2 and
2-to-1.}
\label{fig:mnist_svmw_eq_svmp_toy_w}
\end{figure*}

\begin{figure*}[ht]
\centering
\centerline{\hfill
\includegraphics[width=0.33\textwidth]{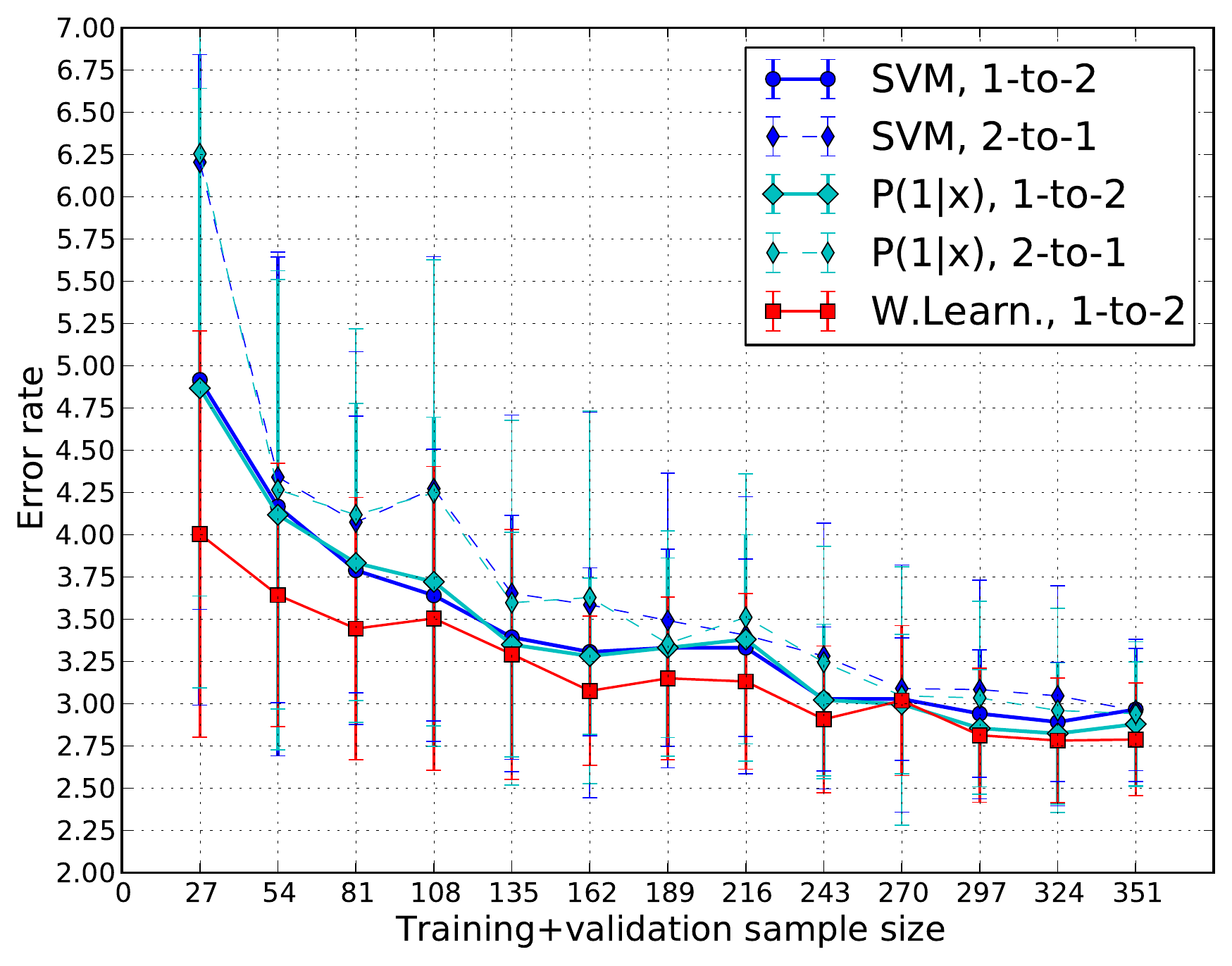}\hfill
\includegraphics[width=0.33\textwidth]{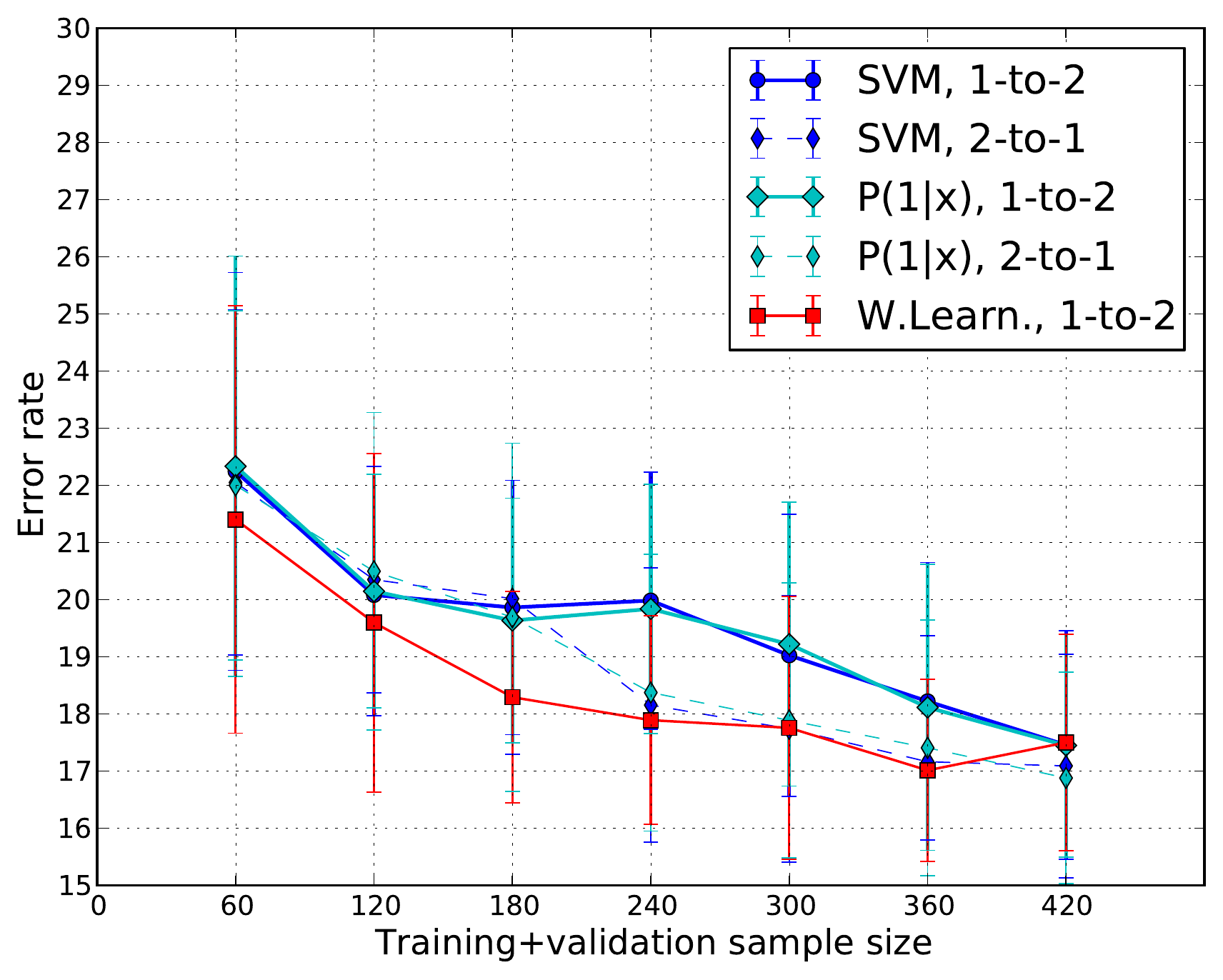}\hfill
\includegraphics[width=0.33\textwidth]{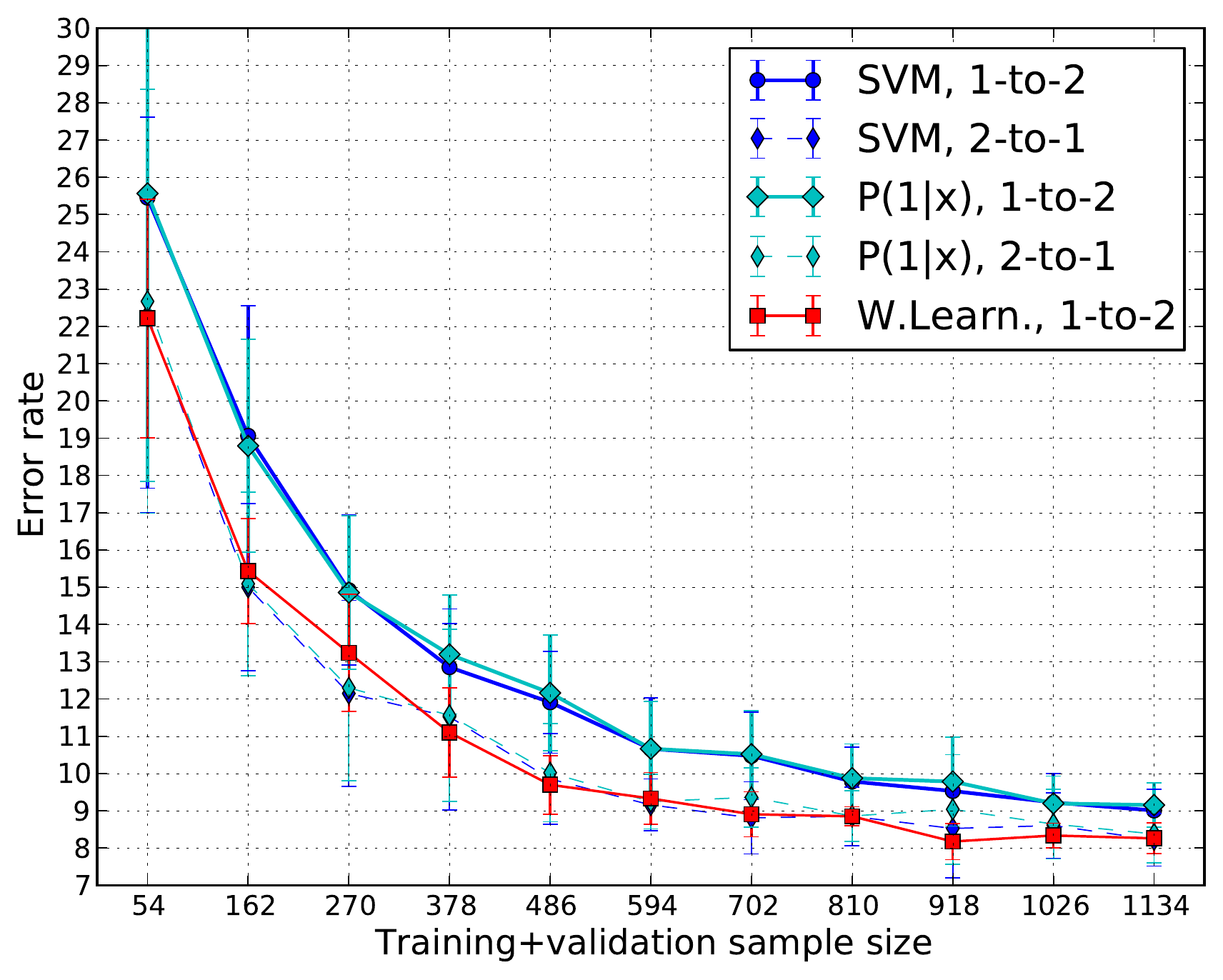}\hfill
}
\caption{SVM and WSVM error rates on the UCI repository data sets with
training-to-validation splits of 1-to-2 and 2-to-1. Left: Breast Cancer
Wisconsin. Middle: Mammographic Mass. Right: Spambase.}
\label{fig:uci_datasets}
\end{figure*}

We begin with the experimental verification of our theoretical findings of
Section~\ref{sec:svmp_svmw_relation}. We reproduced the handwritten digit
recognition experiment of \citet{vapnik:2009:luhi}, where the task is to
discriminate between $5$'s and $8$'s taken from the MNIST database and
downsized to $10 \times 10$ pixels.
We used the features provided by the authors and obtained similar error rates
for both the SVM and the SVM$+$,
see Figure~\ref{fig:mnist_svmw_eq_svmp_toy_w}, left.
Our results are averaged over 100 runs and include more subsets.

The weights for the WSVM algorithm were computed as
$\vc = \valpha^{\star} + \vbeta^{\star}$,
where $\valpha^{\star}$ and $\vbeta^{\star}$ come from the SVM$+$ solution.
We observed that
$\valpha^{\star}_{\rm WSVM} \approx \valpha^{\star}_{\rm SVM+}$.
However, we also observed that, in general,
$b^{\star}_{\rm WSVM} \neq b^{\star}_{\rm SVM+}$,
which is explained by the non-uniqueness of $b$ (Theorem~\ref{thm:svmw_unique}).
If $b^{\star}_{\rm SVM+}$ from the SVM$+$ model is used (WSVM-$b$ in the plot),
then the two classifiers are identical,
but if $b$ is tuned within the constraints imposed by the KKT
conditions (WSVM in the plot), then minor differences appear.

\subsection{Toy data}
\label{sec:exp:toy_data}

We now turn to the problem of choosing weights and evaluate the two
weight generation schemes introduced in Section~\ref{sec:choose_weights}.
In this experiment, data comes from a mixture of 2D Gaussians that form a
non-linear shape resembling ``W'', see Figure~\ref{fig:toy_dataset2}.
Similar to the previous setting, we sample from a fixed training set of size
400, tune parameters, estimate the $\Prob(1|\vx)$, and perform weight learning
on a validation set of size 4000, and test on a separate set of size 2000.
The results are averaged over 50 runs,
see Figure~\ref{fig:mnist_svmw_eq_svmp_toy_w}, middle.
Note that, just like in the experiment of \cite{vapnik:2009:luhi},
this is an idealistic setting where the validation set is so large
that model selection is close to optimal.
In practice, one would never split the available sample as 1-to-40,
therefore we also evaluate more ``reasonable'' splits 1-to-2 and 2-to-1 next.

Figure~\ref{fig:mnist_svmw_eq_svmp_toy_w}, right,
shows results of a similar experiment where the
validation sample is not fixed, but rather obtained by splitting the available
training data.
Since validation samples are now small, the estimation of $\Prob(1|\vx)$
fails and the corresponding WSVM performs on par with the standard SVM.
The weight learning, however, still yields performance improvement on
1-to-2 splits.
Moreover, the WSVM with weight learning is able to achieve a similar error rate
as the SVM trained on twice as much data. We also observe the effect of
overfitting when weight learning is performed on 2-to-1 splits,
and we omit it in further experiments.
Note that one could have anticipated that for the weight learning to
succeed the amount of validation data, in general,
has to be at least comparable to or larger than
the number of weights that are to be learned.

\subsection{UCI data sets}
\label{sec:exp:uci_datasets}

In this set of experiments we evaluate weight learning on three data sets from
the UCI repository \citep{frank:2010:uci}. For every data set, we first remove
any records with missing values and then split the remaining data randomly into
training and test sets of roughly equal size approximately preserving the
initial class distribution.
Table~\ref{tab:uci_stats} summarizes characteristics
of the obtained data sets. Smaller subsets are then sampled from the training
data and split into training and validation sets as 1-to-2 and 2-to-1. The
subsets sampling process is repeated 20 to 50 times depending on the amount of
data. The rest of the experimental setup is the same as before.

\begin{figure*}[ht]
\centering
\centerline{\hfill
\includegraphics[width=0.33\textwidth]{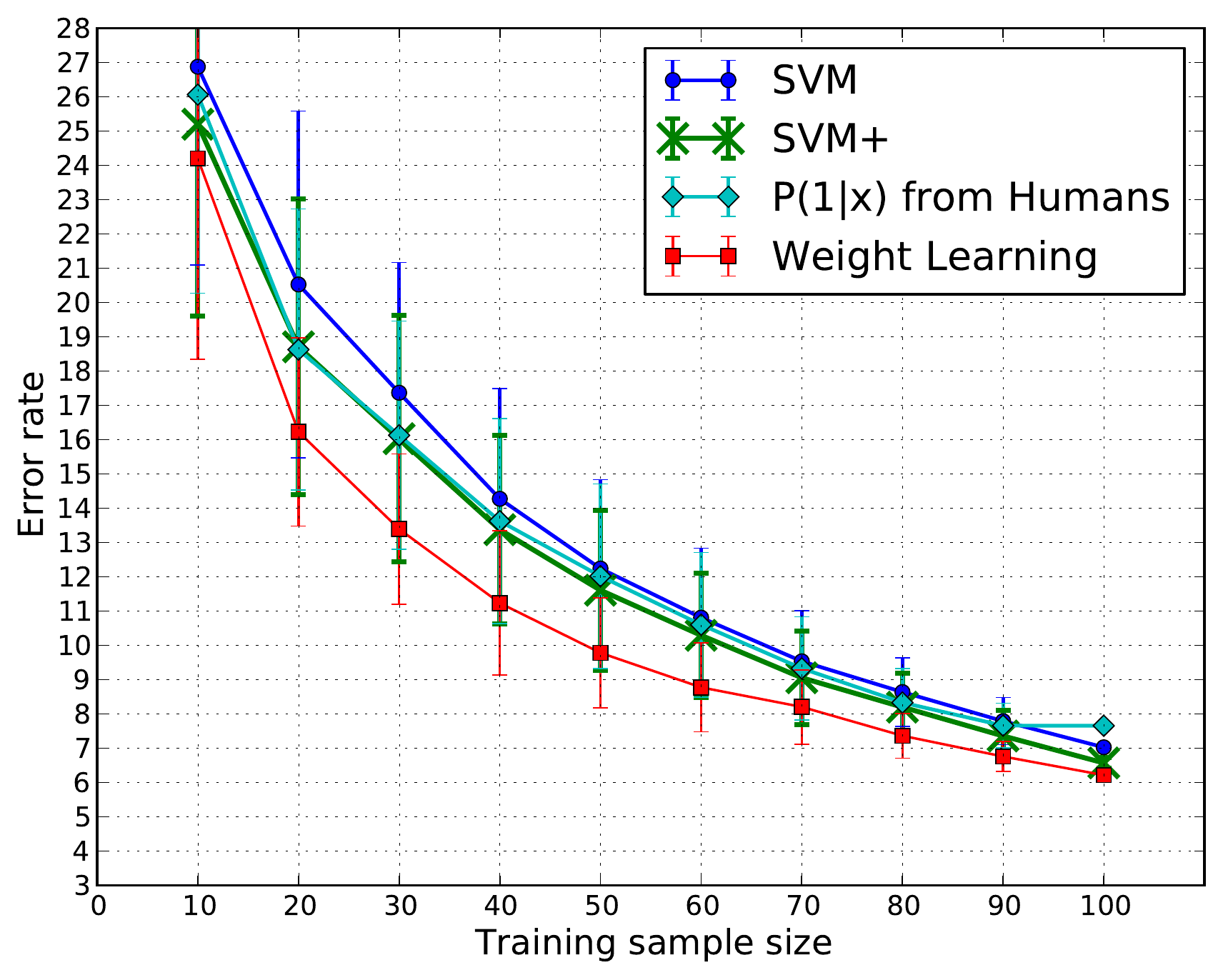}\hfill
\includegraphics[width=0.33\textwidth]{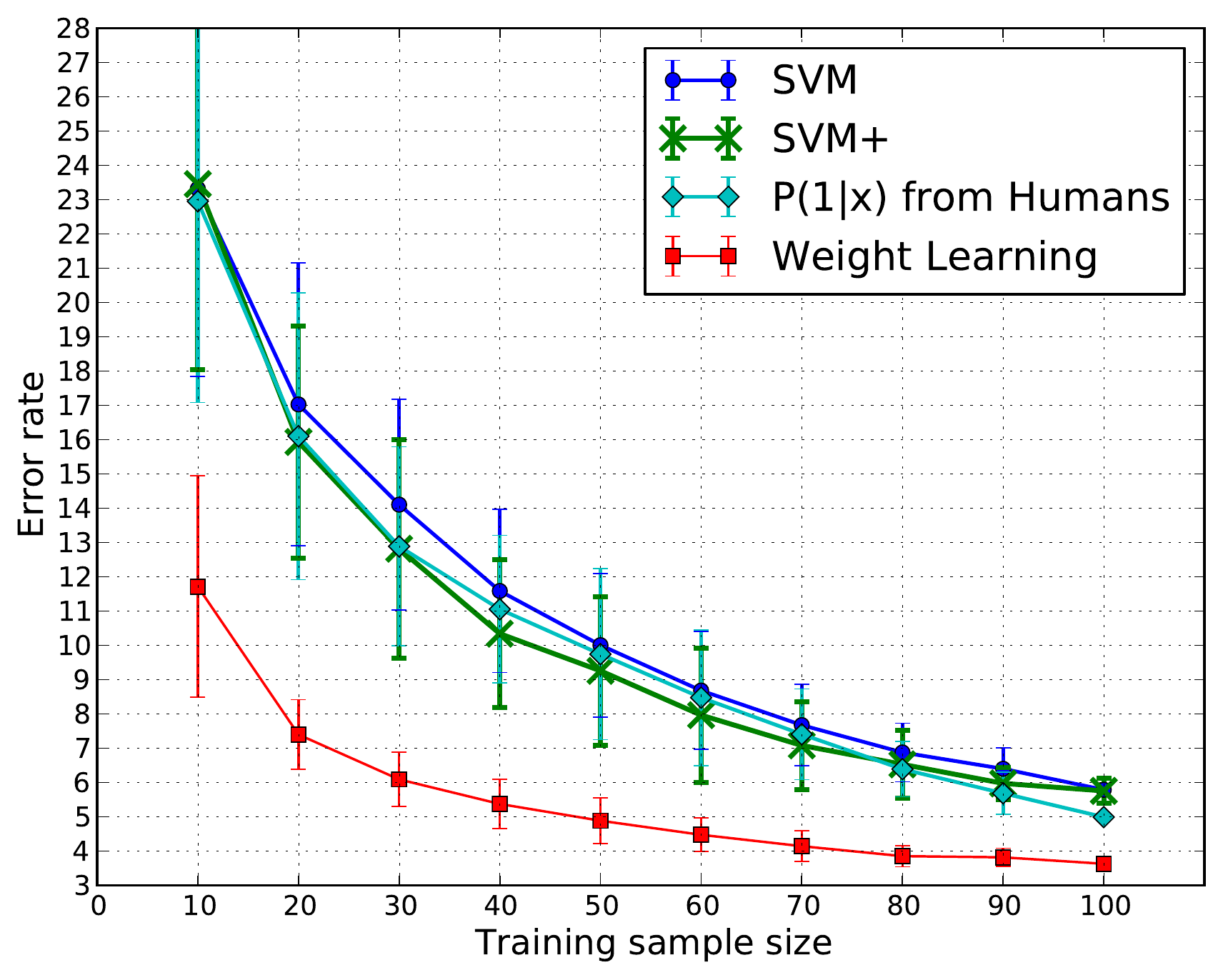}\hfill
}
\caption{Error rate comparison in the handwritten digit recognition experiment
of \cite{vapnik:2009:luhi}. $\Prob(1|\vx)$ was estimated from human rankings.
Left: The original setting. Right: The extended setting where each digit is
translated by 1 pixel in each of the 8 directions.}
\label{fig:mnist_all}
\end{figure*}

\begin{table}[ht]
\caption{Statistics of data sets from the UCI repository.}
\label{tab:uci_stats}
\centering
\begin{small}
\begin{sc}
\begin{tabular}{lccc}
\hline
Data set & Features & Training & Test \\
\hline
BCW & 9 & 351 & 332 \\
Mammographic & 4 & 420 & 410 \\
Spambase & 57 & 2430 & 2171 \\
\hline
\end{tabular}
\end{sc}
\end{small}
\end{table}

\textbf{Breast Cancer Wisconsin (BCW)} \citep{bennett:1992}: On this data set,
the weight learning on the 1-to-2 split performs on par or better than the SVM
on both splits, see Figure~\ref{fig:uci_datasets}. Notably, the SVM performed
worse on the 2-to-1 split, which we attribute to overfitting. The latter is not
too surprising considering the small amount of data and the capacity of the RBF
kernel, which makes the weight learning result even more remarkable.

\textbf{Mammographic Mass} \citep{elter:2007:biopsy}:
Again, the weight learning performs on par or better than
the SVM on all splits for almost all subsets. On the last subset,
however, the weight optimization did not yield any improvement,
and the resulting performance is the same as that of the corresponding SVM.

\textbf{Spambase}: On this data set, the general outcome is that the weight
learning brings roughly the same level of improvement as if twice as much data
were used for training the standard SVM. This can be interpreted as a more
efficient use of training data given the additional knowledge about
importance of each data point.

\subsection{Handwritten digit recognition (5's vs 8's)}
\label{sec:exp:5_vs_8}

Finally, we get back to the original handwritten digit recognition experiment of
\citet{vapnik:2009:luhi} and evaluate our weight generation schemes on that data.

In this experiment, we evaluate the first weight generation scheme
(\ref{eq:wgt:weights_alpha}) under the assumption that digit ranking is
available as the additional information, i.e., in addition to the class label
$\pm 1$, we are also given a confidence score between $-1$ and $1$. This is a
reasonable assumption e.g.\ for data sets where robust annotation is obtained
by aggregation of labeling from several human experts and is similar to
the setting considered by \citet{wu:2004:wmsvmpriorknow}.

We collected additional annotation in the form of ranking from three human
experts. The humans were presented with a random sample of the $10 \times 10$
pixel digits and were asked to label them using one of $5$ possible labels,
which we translated to a score in $\{-1, -0.5, 0, 0.5, 1\}$.
Each of the 100 digits from the training set was ranked 16 times and
the average score was then used as an estimate of $2 \Prob(1|x) - 1$.

Figure~\ref{fig:mnist_all} shows the corresponding experimental results. We
observe that additional information from human experts helps on small subsets,
but its influence degrades on larger subsets.
This might be in part due to the
difference in image representation used by SVMs and humans. In particular,
humans' recognition of digits is translation invariant, while the pixel-wise
representation is not. This leads us to our final experiment on the extended
version of that data set.

We extend the original training sample of 100 digits by shifting each digit
by 1 pixel in all 8 directions, thus obtaining 9 times the initial sample size.
We assume that both the human rankings and the privileged features from the
experiment of \citet{vapnik:2009:luhi} are unaffected by such translations
and we simply replicate them. The experimental results are presented in
Figure~\ref{fig:mnist_all}, right. Note that the WSVM with human rankings is
now consistently on par or better than SVM and is somewhat comparable to
SVM$+$.

Remarkably, weight learning now gives significant performance boost on
the extended version of the data set, which shows that it can be successfully
combined with other sources of additional information, like the hint about
translation invariance in this case.
Interestingly enough, \citet{lauer:2008:prior} discussed the possibility
of combining the virtual sample method, which we used to extend the
training set, with weighted learning where each virtual point would
be given a confidence score $c_i$.
Our weight learning algorithm does exactly that, but without trying to
model the measure of confidence. Instead, it attempts to directly
optimize an estimate of the expected loss $L(f)$.

\section{Conclusion}
\label{sec:conclusion}

We have investigated basic properties of the recently proposed
SVM$+$ algorithm, such as uniqueness of its solution,
and have shown that it is closely related to the well-known
weighted SVM. We revealed that all SVM$+$ solutions are constrained
to have a certain dependency between the dual variables and the incurred
loss on the training sample, and that the prior knowledge from the
SVM$+$ framework can be encoded via instance weights.

That motivated us to consider other sources of additional information
about the training data than the one given by privileged features.
In particular, we considered the weight learning method
in Section~\ref{sec:weight_learning} which allows one to learn
weights directly from data (using a validation set).
The latter approach is not limited to SVMs and can be extended to
other classifiers.

Experimental results confirmed our intuition that importance
weighting is a powerful method of incorporating prior knowledge.
In the idealized setting, we showed that the weight learning works
and yields significant performance improvement. The choice of weights
in a more practical setting is left for future work.

\appendix
\allowdisplaybreaks

\section{The KKT conditions}
\label{sec:kkt}

In convex optimization, the Karush-Kuhn-Tucker (KKT) conditions
are necessary and sufficient for a point to be primal and dual
optimal with zero duality gap \citep{boyd:2004:cvx}.

The KKT conditions corresponding to the weighted SVM problem
(\ref{eq:svmw_primal}) are given below:
\begin{subequations}
\label{eq:svmw_kkt}
\begin{align}
\textstyle\sum_{i=1}^{n} \alpha^{\star}_i y_i \vz_i &= \vw^{\star} ,
    \label{eq:svmw_kkt:w} \\
\textstyle\sum_{i=1}^{n} \alpha^{\star}_i y_i &= 0 ,
    \label{eq:svmw_kkt:b} \\
\alpha^{\star}_i + \beta^{\star}_i &= c_i ,
    \label{eq:svmw_kkt:xi} \\
\alpha^{\star}_i [ \xi^{\star}_i -1
    + y_i (\inner{\vw^{\star}, \vz_i} + b^{\star}) ] &= 0 ,
    \label{eq:svmw_kkt:comp1} \\
\beta^{\star}_i [ \xi^{\star}_i ] &= 0 ,
    \label{eq:svmw_kkt:comp2} \\
\xi^{\star}_i -1 + y_i (\inner{\vw^{\star}, \vz_i} + b^{\star}) &\geq 0 ,
    \label{eq:svmw_kkt:const1} \\
\alpha^{\star}_i \geq 0, \; \beta^{\star}_i \geq 0, \; \xi^{\star}_i &\geq 0 .
    \label{eq:svmw_kkt:const2}
\end{align}
\end{subequations}

And the KKT conditions corresponding to the SVM$+$ problem
(\ref{eq:svmp_primal}) are as follows:
\begin{subequations}
\label{eq:svmp_kkt}
\begin{align}
\textstyle\sum_{i=1}^{n} \alpha^{\star}_i y_i \vz_i &= \vw^{\star} ,
    \label{eq:svmp_kkt:w} \\
\textstyle\sum_{i=1}^{n} \alpha^{\star}_i y_i &= 0 ,
    \label{eq:svmp_kkt:b} \\
\textstyle\sum_{i=1}^{n} (\alpha^{\star}_i + \beta^{\star}_i - C) \cvz_i 
    &= \gamma \cvw^{\star} ,
    \label{eq:svmp_kkt:cw} \\
\textstyle\sum_{i=1}^{n} (\alpha^{\star}_i + \beta^{\star}_i - C) &= 0 ,
    \label{eq:svmp_kkt:cb} \\
\alpha^{\star}_i [ \inner{\cvw^{\star}, \cvz_i} + \cb^{\star} 
    -1 + y_i (\inner{\vw^{\star}, \vz_i} + b^{\star}) ] &= 0 ,
    \label{eq:svmp_kkt:comp1} \\
\beta^{\star}_i [ \inner{\cvw^{\star}, \cvz_i} + \cb^{\star} ] &= 0 ,
    \label{eq:svmp_kkt:comp2} \\
\inner{\cvw^{\star}, \cvz_i} + \cb^{\star}
    -1 + y_i (\inner{\vw^{\star}, \vz_i} + b^{\star}) &\geq 0 ,
    \label{eq:svmp_kkt:const1} \\
\alpha^{\star}_i \geq 0, \; \beta^{\star}_i \geq 0, \;
    \inner{\cvw^{\star}, \cvz_i} + \cb^{\star} &\geq 0 ,
    \label{eq:svmp_kkt:const2}
\end{align}
\end{subequations}

\section{Technical proofs}
\label{sec:proofs}

\subsection{Proof of Theorem~\ref{thm:svmp_unique}}
\label{sec:proof:svmp_unique}
\begin{theorem*}
The solution to the problem (\ref{eq:svmp_primal}) is unique in $(\vw,\cvw,\cb)$
for any $C > 0$, $\gamma > 0$. If there is a support vector, then $b$ is unique
as well, otherwise:
\begin{align*}
\max_{i \in \Ic_{+}} (1 - \inner{\cvw,\cvz_i} - \cb) \leq b \leq 
\min_{i \in \Ic_{-}} (\inner{\cvw,\cvz_i} + \cb - 1) .
\end{align*}
\end{theorem*}
\begin{proof}
Following \citep{burges:1999:unique}, let $F$ be the objective function:
\begin{align*}
F = \frac{1}{2} \norms{\vw}
    + \frac{\gamma}{2} \norms{\cvw}
    + C \sum_{i=1}^{n} (\inner{\cvw, \cvz_i} + \cb) ,
\end{align*}
and define $\vu \bydef (\vw, \cvw, \cb)^{\top}$.
Suppose $\vu_1$ and $\vu_2$ are two
solutions, then, since the problem is convex, a family of solutions is given by
$\vu_t = (1-t) \vu_1 + t \vu_2$, $t \in [0,1]$, and
$F(\vu_1) = F(\vu_2) = F(\vu_t)$.
Expanding $F(\vu_t) - F(\vu_1) = 0$ and differentiating w.r.t.\ $t$ yields:

\begin{align*}
(t-1) \norms{\vw_1} + (1-2t)\inner{\vw_1,\vw_2}
    + t\norms{\vw_2} + \\
\gamma \left[ (t-1) \norms{\cvw_1} + (1-2t)\inner{\cvw_1,\cvw_2}
    + t\norms{\cvw_2} \right] \\
+ t C \sum_{i=1}^{n} (\inner{\cvw_2 - \cvw_1, \cvz_i} + \cb_2 - \cb_1) = 0 , \\
\norms{\vw_1 - \vw_2} + \gamma \norms{\cvw_1 - \cvw_2} = 0 .
\end{align*}
Since $\gamma > 0$ it follows that $\vw_1 = \vw_2$ and $\cvw_1 = \cvw_2$.
Plugging that into the first equation yields $\cb_2 = \cb_1$.
Uniqueness of $b$ now follows from condition (\ref{eq:svmp_kkt:comp1}).
If all $\alpha_i = 0$ (i.e., there are no support vectors),
then $\vw = \zeros$ and the result follows from (\ref{eq:svmp_kkt:const1}).
\end{proof}

\subsection{Uniqueness of the dual solution}
\label{sec:proof:unique_dual}
\begin{proposition}
\label{prop:unique_dual}
If $(\valpha_1, \cvalpha_1)$ and $(\valpha_2, \cvalpha_2)$ are solutions to the
optimization problem (\ref{eq:svmp_dual}), then
\begin{align*}
(\valpha_1 - \valpha_2) &\in \Null(\mY\mK\mY)
    \cap \ones^{\perp} \cap \vy^{\perp} , \\
(\cvalpha_1 - \cvalpha_2) &\in \Null(\cmK) \cap \ones^{\perp}.
\end{align*}
If $\valpha_1$ and $\valpha_2$ are solutions to the problem
(\ref{eq:svmw_dual}), then
\begin{align*}
(\valpha_1 - \valpha_2) \in \Null(\mY\mK\mY)
    \cap \ones^{\perp} \cap \vy^{\perp} .
\end{align*}
\end{proposition}
\begin{proof}
The proof employs the same method as in the proof of
Theorem~\ref{thm:svmp_unique} and we only provide the part
concerning the WSVM problem.

Let $\mK' = \mY\mK\mY$ and consider a family of solutions
$\valpha_t = (1-t) \valpha_1 + t \valpha_2$, $t \in [0,1]$.
Note that $(\valpha_1 - \valpha_2) \in \vy^{\perp}$
follows directly from the optimization constraints.
Expanding $F(\valpha_t) - F(\valpha_1) = 0$ and
differentiating w.r.t.\ $t$ yields:
\begin{align*}
(t-1) \valpha_{1}^{\top} \mK' \valpha_{1}
    + (1-2t) \valpha_{1}^{\top} \mK' \valpha_{2}
    + t \valpha_{2}^{\top} \mK' \valpha_{2} \\
+ \ones^{\top} (\valpha_{1} - \valpha_{2}) = 0 , \\
(\valpha_{1} - \valpha_{2})^{\top} \mK' (\valpha_{1} - \valpha_{2}) = 0 .
\end{align*}
It follows that
$(\valpha_1 - \valpha_2) \in \Null(\mK')$.
Let $\valpha_1 = \valpha_2 + \vv$, $\vv \in \Null(\mK')$, then
from the first equation $\ones^{\top} \vv = 0$, which completes the proof.
\end{proof}

\begin{corollary}
\label{cor:unique_dual}
If $\mK$ has full rank, then solution to the problem (\ref{eq:svmw_dual}) is
unique. If $\mK$ and $\cmK$ have full rank, then solution to the problem
(\ref{eq:svmp_dual}) is unique.
\end{corollary}

\subsection{Proof of Proposition~\ref{prop:equiv_weights}}
\label{sec:proof:equiv_weights}
\begin{proposition*}
Let $(\vw^{\star},b^{\star},\vxi^{\star},\valpha^{\star}, \vbeta^{\star})$ be a
primal-dual optimal point for the WSVM problem (\ref{eq:svmw_primal}).
The point $(\vw^{\star},b^{\star},\vxi^{\star})$
is primal optimal for any weight vector $\vc \in \Wc$,
and all such weights are contained in $\Wc$.
\end{proposition*}
\begin{proof}
The proof consists in a straightforward application of the KKT conditions.
The additional constraint that
$\textstyle\sum_{i} \mu_i = \textstyle\sum_{i} \alpha^{\star}_i$
follows from Proposition~\ref{prop:unique_dual} since it must hold that
$(\vmu - \valpha^{\star}) \in \ones^{\perp}$.
\end{proof}

\subsection{Proof of Theorem~\ref{thm:svmw_from_svmp}}
\label{sec:proof:svmw_from_svmp}
\begin{theorem*}
Let
$(\vw^{\star},b^{\star},\cvw^{\star},\cb^{\star},
\valpha^{\star},\vbeta^{\star})$
be a primal-dual optimal point for the SVM$+$ problem.
There exists a choice of weights $\vc$, namely
$\vc = \valpha^{\star} + \vbeta^{\star}$, and $\vxi^{\star}$ such that
$(\vw^{\star},b^{\star},\vxi^{\star},\valpha^{\star},\vbeta^{\star})$ is a
primal-dual optimal point for the WSVM problem.
\end{theorem*}
\begin{proof}
Given any fixed feasible $\cvalpha$, the SVM$+$ problem (\ref{eq:svmp_dual}) is
equivalent to the WSVM problem (\ref{eq:svmw_dual}) with
$\vc = C\ones + \cvalpha$. In particular, if
$(\valpha^{\star},\cvalpha^{\star})$ is a solution to (\ref{eq:svmp_dual}),
then $\valpha^{\star}$ is a solution to (\ref{eq:svmw_dual}) with
$\vc = C\ones + \cvalpha^{\star} = \valpha^{\star} + \vbeta^{\star}$.
Let $\xi^{\star}_i = \inner{\cvw^{\star}, \cvz_i} + \cb^{\star}$,
then the point
$(\vw^{\star},b^{\star},\vxi^{\star},\valpha^{\star},\vbeta^{\star})$
verifies the KKT conditions for the WSVM problem (\ref{eq:svmw_primal}).
\end{proof}

\subsection{Proof of Lemma~\ref{lem:svmp_necessary}}
\label{sec:proof:svmp_necessary}
\begin{lemma*}
Assume any given $C > 0$, $\gamma \geq 0$ and let
$(\vw^{\star},b^{\star},\cvw^{\star},\cb^{\star},
\valpha^{\star},\vbeta^{\star} )$ be a primal-dual optimal point for
the SVM$+$ problem (\ref{eq:svmp_primal}), then the following holds:
\begin{align}
\label{eq:svmp_necessary:app}
\frac{ \sum_{i=1}^{n} (\alpha^{\star}_i + \beta^{\star}_i) h_i }{ \sum_{i=1}^{n}
(\alpha^{\star}_i + \beta^{\star}_i) } \geq \frac{1}{n} \sum_{i=1}^{n} h_i ,
\end{align}
where $h_i \bydef [1 - y_i (\inner{\vw^{\star}, \vz_i} + b^{\star})]_{+} $
is the hinge loss at a point $i = 1, \ldots, n$.
If $\gamma = 0$, then (\ref{eq:svmp_necessary:app}) is satisfied with equality.
\end{lemma*}
\begin{proof}
It follows from the KKT conditions that
\begin{align*}
\inner{\cvw^{\star}, \cvz_i} + \cb^{\star} & = h_i + \delta_i,
    & \delta_i & \geq 0 , \\
\alpha_{i}^{\star} > 0 \lor \beta_{i}^{\star} > 0 & \Rightarrow
    \delta_i = 0 , &
    (\alpha_{i}^{\star} + \beta_{i}^{\star}) \delta_i & = 0 .
\end{align*}
Multiplying by $(\alpha^{\star}_i + \beta^{\star}_i - C)$
and summing up yields
\begin{align*}
\gamma \inner{\cvw^{\star}, \cvw^{\star}} &= \sum_{i=1}^{n} (\alpha^{\star}_i +
\beta^{\star}_i) h_i - C \sum_{i=1}^{n} (h_i + \delta_i) .
\end{align*}
Note that
$ C = \frac{1}{n} \sum_{i=1}^{n} (\alpha^{\star}_i + \beta^{\star}_i) > 0 $,
hence
\begin{align*}
\gamma \inner{\cvw^{\star}, \cvw^{\star}} &= \sum_{i=1}^{n} (\alpha^{\star}_i +
\beta^{\star}_i) h_i \\
& - \frac{1}{n} \sum_{i=1}^{n} (\alpha^{\star}_i + \beta^{\star}_i)
\sum_{i=1}^{n} (h_i + \delta_i) .
\end{align*}
Since $\gamma \inner{\cvw^{\star}, \cvw^{\star}} \geq 0$, it must hold that
\begin{align}
\sum_{i=1}^{n} (\alpha^{\star}_i + \beta^{\star}_i) h_i 
& \geq \frac{1}{n} \sum_{i=1}^{n} (\alpha^{\star}_i + \beta^{\star}_i)
\sum_{i=1}^{n} (h_i + \delta_i) \notag \\
& \geq \frac{1}{n} \sum_{i=1}^{n} (\alpha^{\star}_i + \beta^{\star}_i)
\sum_{i=1}^{n} h_i . \label{eq:svmp_necessary:ineq}
\end{align}
Division by $\sum_{i=1}^{n} (\alpha^{\star}_i + \beta^{\star}_i)$
completes the proof.
\end{proof}

\subsection{SVM+ reduction to standard SVM}
\label{sec:proof:svmp_rho_equality}
We show that when there is an equality in the previous lemma,
then the SVM$+$ can be reduced to the standard SVM. For simplicity,
we only state this result in the linear setting, where
$\cvx_i \in \Rb^d$.
\begin{proposition}
\label{prop:svmp_rho_equality}
Assume the setting of Lemma~\ref{lem:svmp_necessary} and let
(\ref{eq:svmp_necessary:app}) be satisfied with equality, then
\begin{align*}
\inner{\cvw^{\star}, \cvx_i} + \cb^{\star} & = h_i, &
    i & = 1, \ldots, n
\end{align*}
Furthermore, the following holds.

1. If $\gamma > 0$, then $\cvw^{\star} = \zeros$ and $\cb^{\star} = h_i$
for all $i$, i.e.,
the loss on all data points is the same and the hard margin SVM is a
special case with $\cb^{\star} = 0$.

2. If $\gamma = 0$, then $\cmX \cvalpha^{\star} = \zeros$, where
$\cmX$ is the matrix of $\cvx_i$ stacked.
If additionally $\rank(\cmX) = n$,
then $\alpha^{\star}_i + \beta^{\star}_i = C$ for all $i$ and
any vector in $\Rb^n$ can be represented via
$\inner{\cvw^{\star}, \cvx_i} + \cb^{\star}$,
hence the soft margin SVM is recovered with
$\xi^{\star}_i = \inner{\cvw^{\star}, \cvx_i} + \cb^{\star}$.
\end{proposition}
\begin{proof}
It follows from (\ref{eq:svmp_necessary:ineq}) that $\delta_i = 0$ and
$\inner{\cvw^{\star}, \cvx_i} + \cb^{\star} = h_i$ for $i = 1, \ldots, n$.
If $\gamma > 0$, then $\gamma \inner{\cvw^{\star}, \cvw^{\star}} = 0$
implies $\cvw^{\star} = \zeros$ and thus
$\cb^{\star} = h_i$ for all $i$.

If $\gamma = 0$, then (\ref{eq:svmp_kkt:cw}) implies
$\cmX \cvalpha^{\star} = \zeros$, where
$\cvalpha^{\star} = \valpha^{\star} + \vbeta^{\star} - C \ones$, as before.
If $\rank(\cmX) = n$, then $\cmX \cvalpha^{\star} = \zeros$
yields $\cvalpha^{\star} = \zeros$,
and so $\alpha^{\star}_i + \beta^{\star}_i = C$ for $i = 1, \ldots, n$.
Since $(\cvx_i)_{i=1}^{n}$ is in this case a basis in $\Rb^{n}$
and there is no penalty on $\norms{\cvw}$ in the objective function,
the SVM$+$ does not impose any additional constraints compared to
the soft margin SVM. The primal-dual optimal point of the SVM$+$ is thus
also optimal for the SVM with
$\xi_{i}^{\star} = \inner{\cvw^{\star}, \cvx_i} + \cb^{\star}$.
\end{proof}

\subsection{Proof of Theorem~\ref{thm:svmw_to_svmp_iff}}
\label{sec:proof:svmw_to_svmp_iff}
\begin{theorem*}
Let
$(\vw^{\star},b^{\star},\vxi^{\star},\valpha_{0}^{\star},\vbeta_{0}^{\star})$
be a primal-dual optimal point for the WSVM problem
with instance weights $\vc_0 \in \Rb_{+}^{n}$, not all zero.
There exists a choice of $C$, $\gamma$,
and correcting features $\{\cvx_i\}_{i=1}^{n}$
such that $(\vw^{\star},b^{\star})$ is optimal for the SVM$+$ problem iff:
\begin{align}
\label{eq:svmw_to_svmp_iff:app}
\exists \: \vc \in \Wc \; : \;
\rho(\vc, \vxi^{\star}) \bydef
    \inner{\vc - \bar{c} \ones, \vxi^{\star}} \geq 0 ,
\end{align}
where $\bar{c} \bydef n^{-1} \sum_{i=1}^{n} c_i$.
If $\rho(\vc, \vxi^{\star}) \geq 0$, one such possible choice is as follows:
\begin{align}
\label{eq:svmw_to_svmp_feat:app}
C &= \bar{c} , & 
\gamma &= \rho(\vc, \vxi^{\star}) , &
\cx_i &= \xi^{\star}_{i} - \cb^{\star} , \; \forall i
\end{align}
moreover, the optimal $\cw^{\star}$ and $\cb^{\star}$ in that case are:
\begin{align}
\label{eq:svmw_to_svmp_copt:app}
\cw^{\star} &= 1, & \cb^{\star} &= \inner{\vc, \vxi^{\star}} / \inner{\vc,
\ones} .
\end{align}
\end{theorem*}
\begin{proof}
\textbf{(\ref{eq:svmw_to_svmp_iff:app}) is necessary.}
Assume there exists an SVM$+$ setting such that
$(\vw^{\star},b^{\star},\cvw^{\star},\cb^{\star},
\valpha^{\star},\vbeta^{\star} )$
is a primal-dual optimal point for the SVM$+$ problem
(\ref{eq:svmp_primal}) and let
$\vc = \valpha^{\star} + \vbeta^{\star}$
(note that $(\valpha^{\star},\vbeta^{\star})$ and
$(\valpha_{0}^{\star},\vbeta_{0}^{\star})$ may be different).
Theorem~\ref{thm:svmw_from_svmp} states that
there exists $\vxi_{0}^{\star}$ such that
$(\vw^{\star},b^{\star},\vxi_{0}^{\star},
\valpha^{\star},\vbeta^{\star})$
is primal-dual optimal for the WSVM problem with weights $\vc$.
We need to show that $\vxi_{0}^{\star} = \vxi^{\star}$.
This follows directly from the KKT conditions when all
$c_{0,i} > 0$ and $c_i > 0$ since
$h_i = [1 - y_i(\inner{\vw^{\star}, \vx_i} + b^{\star})]_{+}$
are the same for both problems.
If some of the weights are zero, then the corresponding
$\xi_{i}^{\star}$ is not uniquely defined (it is unbounded from above)
and we have to assume that the algorithm returns the value at the lower bound,
i.e., $\xi_{i}^{\star} = h_i$.
Now, given that $\vxi_{0}^{\star} = \vxi^{\star}$,
$\vc \in \Wc$ by Proposition~\ref{prop:equiv_weights}
and $\rho(\vc, \vxi^{\star}) \geq 0$ by
Corollary~\ref{cor:svmp_from_svmw_necessary}.

\textbf{(\ref{eq:svmw_to_svmp_iff:app}) is sufficient.}
First, consider the case $\rho(\vc, \vxi^{\star}) > 0$ and let
$(\vw^{\star},b^{\star},\vxi^{\star},\valpha^{\star},\vbeta^{\star})$
be a primal-dual optimal point of the WSVM problem with weights $\vc$.
We construct $\{\cvx_i\}_{i=1}^{n}$ and provide
$C > 0$, $\gamma > 0$, $\cvw^{\star}$, and $\cb^{\star}$ such that
$(\vw^{\star},b^{\star},\cvw^{\star},\cb^{\star},
\valpha^{\star},\vbeta^{\star} )$
is primal-dual optimal for the corresponding SVM$+$ problem.

It is sufficient to look for one dimensional correcting features that
additionally satisfy $\sum_{i=1}^{n} c_i \cx_i = 0 $.
The KKT conditions in this case imply that
\begin{align}
\label{eq:svmw_to_svmp_iff:cw_c}
\cw^{\star} &= - \frac{C}{\gamma} \sum_{i=1}^{n} \cx_i , &
C &= \frac{1}{n} \sum_{i=1}^{n} c_i = \bar{c} .
\end{align}
We require for all $i = 1, \ldots, n$ that
\begin{align}
\label{eq:svmw_to_svmp_iff:xi}
\cw^{\star} \cx_i + \cb^{\star} = [1 - y_i (\inner{\vw^{\star}, \vx_i} +
b^{\star})]_{+} = \xi^{\star}_i.
\end{align}
Multiplying both sides by $c_i$ and summing up yields
\begin{align*}
\cb^{\star} = \inner{\vc, \vxi^{\star}} / \inner{\vc, \ones} .
\end{align*}
Plugging (\ref{eq:svmw_to_svmp_iff:cw_c}) into (\ref{eq:svmw_to_svmp_iff:xi})
and solving for $\cx_i$ one gets:
\begin{align}
\label{eq:svmw_to_svmp_iff:cx}
\cx_i = \pm \sqrt{ \frac{\gamma}{\rho(\vc, \vxi^{\star})} } (\cb^{\star} -
\xi^{\star}_i) .
\end{align}
Choosing $\gamma = \rho(\vc, \vxi^{\star})$ and the plus sign in
(\ref{eq:svmw_to_svmp_iff:cx}) for convenience,
(\ref{eq:svmw_to_svmp_iff:cw_c}) leads to
$\cw^{\star} = \rho(\vc, \vxi^{\star}) / \gamma = 1$.

Now, consider $\rho(\vc, \vxi^{\star}) = 0$.
Let $\cmX = [\cvx_1 \cdots \cvx_n]$ and set $\gamma = 0$.
Proposition~\ref{prop:svmp_rho_equality} and the KKT conditions imply:
\begin{align*}
C &= \bar{c} , & \cmX (\vc - \bar{c} \ones) &= \zeros, & \cmX^{\top}
\cvw^{\star} + \cb^{\star} \ones &= \vxi^{\star}.
\end{align*}
Hence, the matrix $\cmX$ must satisfy
\begin{align*}
(\vc - \bar{c} \ones) &\in \Null(\cmX) , & (\vxi^{\star} - \cb^{\star} \ones)
&\in \Range(\cmX^{\top}) .
\end{align*}
The above requirements translate to
\begin{align*}
\innern{\vc - \bar{c}, \vxi^{\star} - \cb^{\star} \ones} = 0 ,
\end{align*}
which holds for (\ref{eq:svmw_to_svmp_feat:app}), (\ref{eq:svmw_to_svmp_copt:app}),
and $\rho(\vc, \vxi^{\star}) = 0$.
\end{proof}

\subsection{Proof of Proposition~\ref{prop:svmw_to_svmp_iff_easy}}
\label{sec:proof:svmw_to_svmp_iff_easy}
\begin{proposition*}
Let
$(\vw^{\star},b^{\star},\vxi^{\star},\valpha^{\star},\vbeta^{\star})$
be a primal-dual optimal point for the WSVM problem
with instance weights $\vc \in \Rb_{+}^{n}$, not all zero. If
\begin{align*}
\Null(\mY\mK\mY) \cap \ones^{\perp} \cap \vy^{\perp} = \{\zeros\} ,
\end{align*}
then there exists a choice of $C$, $\gamma$, and $\{\cvx_i\}_{i=1}^{n}$
such that $(\vw^{\star},b^{\star})$ is optimal for the SVM$+$ problem iff:
\begin{align*}
\rho(\valpha^{\star}, \vxi^{\star}) =
    {\vxi^{\star}}^{\top} \big( \Id - \tfrac{1}{n}\ones\ones^{\top} \big)
    \valpha^{\star} \geq 0 .
\end{align*}
\end{proposition*}
\begin{proof}
Sufficiency follows directly from Theorem~\ref{thm:svmw_to_svmp_iff}
since $\vc = \valpha^{\star}$ is a valid choice of weights
(cf.\ Definition~\ref{def:equiv_weights}).
For necessity, note that $\valpha^{\star}$ is unique
by Proposition~\ref{prop:unique_dual} and all weights in $\Wc$
are of the form $\vc = \valpha^{\star} + \vbeta$, $\vbeta \in \Vc$.
The maximum in (\ref{eq:svmw_to_svmp_iff}) corresponds to
\begin{align*}
\max_{\vbeta} \sum_{i=1}^{n} \xi^{\star}_i \beta_i
- \frac{1}{n} \sum_{i=1}^{n} \xi^{\star}_i \sum_{i=1}^{n} \beta_i ,
\quad \st \; \beta_i \geq 0 ,
\end{align*}
which is attained at $\vbeta = \zeros$ since
$\forall i \; \xi^{\star}_i \beta_i = 0$.
\end{proof}

\subsection{Proof of Theorem~\ref{thm:derivatives}}
\label{sec:proof:derivatives}
\begin{theorem*}
Let the loss function $\loss$ be convex and twice continuously differentiable
and let the kernel matrix $\mK$ be (strictly) positive definite. Define vectors
$\vu$ and $\vv$ componentwise for $i = 1, \ldots, n$ as
\begin{align*}
u_i &\bydef y_i \loss'(y_i [K_{i}^{\top} \valpha^{\star} + b^{\star}]) , \\
v_i &\bydef c_i \loss''(y_i [K_{i}^{\top} \valpha^{\star} + b^{\star}]) ,
\end{align*}
where $(\valpha^{\star}, b^{\star})$ is a solution of (\ref{eq:wlearning:fopt})
for a given $\vc$.
If $\vv \neq \zeros$, then the solution is unique,
$\valpha^{\star}$ and $b^{\star}$ are continuously differentiable w.r.t.\ $\vc$
and the corresponding gradient can be computed as follows.
\begin{align}
\label{eq:gradient:app}
\begin{bmatrix}
\frac{\partial \valpha^{\star}}{\partial \vc} \\
\frac{\partial b^{\star}}{\partial \vc}
\end{bmatrix}
= -
\begin{bmatrix}
\Id + \diag(\vv) \mK & \vv \\
\ones^{\top} & 0
\end{bmatrix}^{-1}
\begin{bmatrix}
\diag(\vu) \\
\zeros^{\top}
\end{bmatrix}
\end{align}
\end{theorem*}
\begin{proof}
Uniqueness of solution follows from a similar argument as in the proof of
Theorem~\ref{thm:svmp_unique} and is obvious for $\valpha$. Let $b_{1}^{\star}$
and $b_{2}^{\star}$ be two optimal $b$ and define $b_{t}^{\star} =
(1-t)b_{1}^{\star} + t b_{2}^{\star}$. Considering the difference of the
objective function at $b_{t}^{\star}$ and $b_{1}^{\star}$ and differentiating
twice w.r.t.\ $t$, one arrives at
\begin{align*}
(b_{2}^{\star} - b_{1}^{\star}) \ones^{\top} \vv = 0  \quad \Rightarrow \quad
b_{2}^{\star} = b_{1}^{\star} .
\end{align*}

The optimality conditions of (\ref{eq:wlearning:fopt}) yield
\begin{align*}
\mK(\valpha^{\star} + \diag(\vu) \vc) &= \zeros , &
\inner{\vu, \vc} &= 0 .
\end{align*}
Since $\mK$ is non-singular it can be dropped from the first equation.
Computation of the total derivatives yields the linear system below.
\begin{align}
\label{eq:deriv:system}
\begin{bmatrix}
\Id + \diag(\vv) \mK & \vv \\
\vv^{\top} \mK & \ones^{\top} \vv
\end{bmatrix}
\begin{bmatrix}
\frac{\partial \valpha^{\star}}{\partial \vc} \\
\frac{\partial b^{\star}}{\partial \vc}
\end{bmatrix}
= -
\begin{bmatrix}
\diag(\vu) \\
\vu^{\top}
\end{bmatrix}
\end{align}
Note that (\ref{eq:deriv:system}) is equivalent to the system in
(\ref{eq:gradient:app}) since the last equation
can be equivalently replaced by the
sum of the first $n$ equations minus the last one. To apply the implicit
function theorem, it remains to show that the matrix
in (\ref{eq:gradient:app}) is
invertible. Recall that the determinant of a block matrix factors as the
determinant of a block and its Schur complement.
It is thus sufficient to show that
\begin{align*}
\det(\Id + \diag(\vv) \mK) \neq 0, \;
\ones^{\top} (\Id + \diag(\vv) \mK)^{-1} \vv \neq 0 .
\end{align*}
Assume w.l.o.g.\ that the first $m \leq n$ components of $\vv$ are non-zero and
define $M \bydef \Id + \diag(\vv) \mK$. Further,
\begin{align*}
M = 
\begin{bmatrix}
A & B \\
C & D
\end{bmatrix}
=
\begin{bmatrix}
\Id_m + \diag(\vv_m) \mK_m & B \\
\zeros_{n-m,m} & \Id_{n-m}
\end{bmatrix} ,
\end{align*}
where the $B$ block is irrelevant. It now follows that
\begin{align*}
& \det(M) \\
& \quad = \det(D) \det(A - B D^{-1} C) = \det(A) \\
& \quad = \det(\diag(\vv_m)) \det(\diag(\vv_m)^{-1} + \mK_m) \neq 0 ,
\end{align*}
where we use that $\diag(\vv_m)^{-1}$ is positive definite since $\vv_m \succ
\zeros_m$ due to convexity of $\loss$. Finally,
\begin{align*}
\ones^{\top} M^{-1} \vv &=
\ones^{\top}
\begin{bmatrix}
A^{-1} & -A^{-1} B D^{-1} \\
\zeros_{n-m,m} & D^{-1}
\end{bmatrix}^{-1} \vv \\
&=
\ones_{m}^{\top}
\left(\diag(\vv_m)^{-1} + \mK_m \right)^{-1}
\ones_{m} > 0 .
\end{align*}
\end{proof}

{\small
\bibliographystyle{icml2013}
\bibliography{main}
}

\end{document}